\documentclass{article}

\usepackage{PRIMEarxiv}

\usepackage[utf8]{inputenc} 
\usepackage[T1]{fontenc}    

\usepackage{amsmath,amssymb,amsthm,mathtools,bm}
\usepackage{enumitem}
\usepackage{microtype}
\usepackage{algorithm}
\usepackage{algpseudocode}
\usepackage[numbers,sort&compress]{natbib}

\usepackage{dblfloatfix}
\newtheorem{theorem}{Theorem}
\newtheorem{lemma}{Lemma}
\newtheorem{assumption}{Assumption}
\newtheorem{corollary}{Corollary}

\newtheorem{remark}{Remark}

\newcommand{\R}{\mathbb R}
\newcommand{\E}{\mathbb E}
\newcommand{\Pp}{\mathbb P}
\newcommand{\F}{\mathcal F}
\newcommand{\norm}[1]{\left\lVert #1\right\rVert}
\newcommand{\ip}[2]{\left\langle #1,#2\right\rangle}
\newcommand{\ipF}[2]{\left\langle #1,#2\right\rangle_F}

\newcommand{\op}{\mathrm{op}}
\newcommand{\M}{\mathcal M}
\newcommand{\cU}{\mathcal U}
\newcommand{\KL}{D_{\mathrm{KL}}}

\usepackage[
    colorlinks=true,
    citecolor=blue,
    linkcolor=blue,
    urlcolor=blue
]{hyperref}
\usepackage[nameinlink,capitalize]{cleveref}
\DeclareMathOperator*{\argmin}{arg\,min}

\title{Minimax Optimality of Shampoo
for Nonstationary Stochastic Optimization
}

\author{
  Abir Sarkar$^{\ast}$ \\
  Department of Statistics and Data Science\\
  Cornell University\\
  \texttt{as4458@cornell.edu} 
   \And
  Nirmalya Mandal$^{\ast}$ \\
  Department of Statistics\\
  NC State University\\
  \texttt{nmandal@ncsu.edu} 
   \And
  Martin T. Wells \\
  Department of Statistics 
  and Data Science\\
  Cornell University\\
  \texttt{mtw1@cornell.edu}
}

\begin{document}
\maketitle
\begingroup
\renewcommand{\thefootnote}{\fnsymbol{footnote}}
\footnotetext[1]{Equal contribution.}
\endgroup

\begin{abstract}
We develop a comprehensive theoretical analysis of
Matrix--\textbf{Shampoo} for tracking time-varying optima in
nonstationary stochastic optimization. We derive finite-horizon
high-probability bounds for average tracking error that disentangle initialization,
optimizer drift, momentum-induced drift penalty, stochastic noise, and adaptive
preconditioning, with explicit dependence on momentum parameter
\(\beta_1\) and stabilization parameter \(\varepsilon\).
These bounds reveal an explicit drift--noise  and stability--geometry tradeoffs: 
noise-dominated regimes favor larger $\beta_1$, which provides stronger first-moment averaging and greater suppression of gradient noise, whereas drift-dominated regimes favor smaller $\beta_1$, allowing faster adaptation to moving optimizer. Larger $\varepsilon$ reduces the preconditioner-perturbation
term while weakening Shampoo's ability to exploit matrix geometry, thereby reducing adaptivity. We establish a matching minimax lower bound under squared path-variation constraint, proving the information-theoretic optimality of Shampoo over the class of causal optimization procedures. Under a constant stepsize, we establish
high-probability bounds for the average projected stationarity gap along general nonstationary nonconvex objectives. We validate our theoretical findings
on both synthetic and multiple real-world problems.
\end{abstract}



\section{Introduction}

We study stochastic optimization over a compact convex set \(\Theta\subset\mathbb{R}^d\), where the objective evolves over time according to
\begin{equation}
    \tag{Opt}\label{eq:nonstationary-convex-opt}
    \boldsymbol{\theta}_{t}^{\star} \in \argmin_{\boldsymbol{\theta}\in \Theta} G_{t}(\boldsymbol{\theta}), \;\; G_{t}(\boldsymbol{\theta}) = \mathbb{E}_{X_{t} \sim \Pi_{t}}[g(\boldsymbol{\theta}, X_{t})].
\end{equation}

Here \(g(\boldsymbol{\theta},X_t)\) is a noisy sample-level version of \(G_t(\boldsymbol{\theta})\), with \(X_t\sim\Pi_t\), where \(\Pi_t\) is a time-varying distribution. As \(\Pi_t\) changes, so does the population objective \(G_t\) and its minimizer \(\boldsymbol{\theta}_t^\star\). Such problems arise in signal processing \cite{Kushner1997,sayed2003fundamentals}, concept drift and online learning \cite{10.1561/2400000013,NEURIPS2021_62e7f2e0}, and minimum divergence-type statistical estimation objectives \cite{roy2026asymptotic}, where the goal is to follow an evolving optimizer rather than converge to a fixed solution.

Adaptive gradient methods such as Adam \cite{DBLP:journals/corr/KingmaB14} are standard tools in large-scale deep learning \cite{brown2020language, dosovitskiy2020image}, combining momentum with coordinate-wise adaptive step sizes. However, diagonal preconditioning treats parameters independently and can fail to exploit the parameter space geometry \cite{zhang2026understanding}, motivating richer matrix-valued adaptive methods \cite{duchi2011adaptive}. Shampoo \cite{pmlr-v80-gupta18a} addresses this limitation  through computationally tractable left and right-sided preconditioners that exploit the matrix or tensor structure of the parameters \cite{eschenhagen2026purifying}. In stationary settings, Shampoo enjoys favorable convergence guarantees by leveraging Kronecker-factored preconditioning \cite{pmlr-v80-gupta18a,vyas2025soap,xie2025structured}. Very recently, Shampoo has been refined through KL-based structured geometry and unified analyses of one and two-sided preconditioning \cite{lin2025understanding,li2026convergence}.

In nonstationary optimization, historical gradients create an inherent tension between noise reduction and adaptation to a moving objective \cite{JMLR:v24:21-1410,sahu2026provable,sahu2026adapt,bernardi2026performance}. This raises the first fundamental question: \textit{does this noise--adaptation tradeoff persist when first-moment memory is coupled with Shampoo’s two-sided matrix geometry?}

A second difficulty is that Shampoo’s preconditioner is  constructed from current and past stochastic gradients and is therefore nonpredictable. For analytically tractable results, existing literature often introduces predictable proxy preconditioners and assumes adaptive strong monotonicity \cite{sahu2026adapt}. Hence, the second question: \textit{can one analyze Shampoo's
realized preconditioner directly, while obtaining asymptotic minimax-optimal guarantees?}

 Empirical evidence links stale optimizer to reduced plasticity under task changes  \cite{lyle2023understanding,lee2024slow}, raising a fundamental stability--geometry question: \textit{How should the ridge balance robustness to noisy or stale second-moment estimates against the benefits of matrix adaptation?} 
 
 These questions motivate our theoretical study of when, and under what conditions, Shampoo’s matrix adaptivity yields sharp guarantees under distribution shift.

\subsection{Our Contributions}

\begin{enumerate}
    \item \textbf{Finite-time tracking Error Bound.}
    We provide the first finite-horizon tracking guarantees for Shampoo when the population optimizer evolves over time. Our high-probability bound on the average tracking error
    \(
        \frac{1}{T}\sum_{t=1}^T
        \|\theta_t-\theta_t^\star\|^2
    \)
    separates the contributions of initialization, optimizer drift, momentum-induced drift, stochastic-gradient noise, and matrix-preconditioner perturbations. The main challenge is the coupled, nonpredictable dependence of momentum and the Shampoo preconditioner on stochastic gradients. Our tracking error analysis isolates their respective contributions in the final bound.

    \item \textbf{Minimax-optimal tracking under squared path variation.}
    We measure nonstationarity through the squared path-variation budget
    \(
        \mathcal D_T
        :=
        \sum_{t=1}^{T}
        \|\theta_{t+1}^\star-\theta_t^\star\|^2,
    \)
    which controls the cumulative movement of the optimizer without imposing a pointwise bound on individual changes. Over this class of drifting problems, we establish an asymptotic minimax lower bound for arbitrary causal optimization procedures and show that appropriately tuned Shampoo attains the same rate,
    \[
        r_T^\star
        \asymp
        \frac{\sigma^{4/3}\mathcal D_T^{1/3}}
        {\mu^{4/3}T^{1/3}},
    \]
    up to constant factors. Here $\sigma$ denotes the gradient noise level, and we assume $\mu$-strong monotonicity on the mean-gradient map. The matching lower bound shows that no causal optimization rule can improve the asymptotic dependence on horizon $T$, and drift budget $\mathcal D_T$.

\item \textbf{Noise--drift and ridge--geometry tradeoffs.}
Our bound reveals two distinct effects of Shampoo's memory and ridge parameters. First, noise-dominated regimes support larger $\beta_1$ through stronger averaging, while drift-dominated regimes favor smaller $\beta_1$ for faster adaptation to the moving optimizer. Related phenomena have been established for momentum SGD and coordinatewise adaptive methods \cite{sahu2026provable,sahu2026adapt, zheng2026analysis, defossez2022simple}; our analysis extends this picture to Shampoo, where first-moment memory is coupled with a matrix-valued preconditioner. Second, the ridge parameter $\varepsilon$ controls the strength of this matrix adaptation. Larger $\varepsilon$ tightens the preconditioner-perturbation
bound but simultaneously drives the normalized Shampoo preconditioner toward the identity,  weakening matrix geometry adaptation. Smaller $\varepsilon$ preserves stronger matrix preconditioning and reduces the effective preconditioned noise energy, but enlarges the perturbation bound. This stability--adaptivity tension is consistent with empirical evidence of stale optimizer states and degraded adaptation under rapid distribution shifts \cite{bai2025adaptive,gogianu2026revisiting}, as well as Shampoo studies on damping and stale preconditioners \cite{mei2023kradagrad,eschenhagen2026purifying,nam2026foam}.

\item \textbf{Projected stationarity beyond strong monotonicity.}
Without imposing strong monotonicity, we derive a high-probability bound on the average projected stationarity gap for Shampoo along general nonstationarynonconvex objectives. Temporal objective variation, momentum-induced bias, filtered stochastic noise, and distortion from the evolving preconditioner all enter explicitly, showing that the qualitative results identified in the tracking analysis persist in the nonconvex setting.

    \item \textbf{Numerical validation across structured nonstationary problems.}
 Synthetic experiments on convex and nonconvex problems validate the predicted momentum-noise and ridge-geometry tradeoffs and support the minimax tracking scaling across multiple path-variation regimes. On \textbf{CLEAR-10} and \textbf{CORe50}, comparisons with SGD, momentum SGD,
and Adam demonstrate Shampoo's stable performance across
 momentum choices under real-world distribution shifts.
\end{enumerate}

\subsection{Related Work}

For a matrix parameter $W\in\mathbb R^{m\times n}$, full-matrix AdaGrad \cite{duchi2011adaptive} captures cross-coordinate geometry with an $mn\times mn$ preconditioner, incurring $O(m^2n^2)$ memory and $O(m^3n^3)$ inverse-root cost. Shampoo \cite{pmlr-v80-gupta18a} exploits the matrix structure of $W$ through $m\times m$ and $n\times n$ second-moment factors, reducing these costs to $O(m^2+n^2)$ and $O(m^3+n^3)$ while preserving matrix-valued geometry. Recent work further clarifies the structure captured by this geometry \cite{morwani2025new}.

A substantial literature has improved the scalability and practical behavior of Shampoo. Distributed implementations use blocking, grafting \cite{agarwal2020disentangling}, infrequent inverse-root updates, and parallelization \cite{anil2020scalable,shi2023distributed}, while low-rank, sketching, quantization, and faster inverse-root methods further reduce memory and computation \cite{agarwal2019efficient,feinberg2023sketchy,wang24054,modoranu2026dash}. Alternative Kronecker constructions and structured approximations further refine the tradeoff between computational cost and matrix geometry \cite{mei2023kradagrad,xie2025structured}. More recent work improves Shampoo's geometry through Adam-style updates in its spectral basis \cite{vyas2025soap}, eigenvalue and eigenbasis corrections \cite{eschenhagen2026purifying}, and alternative covariance approximations \cite{lin2025understanding}. Adaptive damping and stabilization of infrequently updated preconditioners further address stale second-moment information \cite{nam2026foam,frans2026stable}. Recent extensions combine Shampoo geometry with richer momentum schemes, as in \cite{yuan2411mars}, while large-scale benchmarks increasingly evaluate Shampoo-derived methods alongside AdamW, Muon, and other modern optimizers \cite{semenov2509benchmarking}. These advances are substantial, but a significant part of the accompanying theory is predominantly developed under stationarity.

In parallel, a broad literature studies optimization under time-varying objectives, with dynamic-regret and tracking guarantees under variation or path-length constraints \cite{besbes2015non,yang2016tracking,JMLR:v24:21-1410,xie2026dynamic}. More recent work studies how momentum and adaptive moment estimation behave under nonstationarity \cite{sahu2026provable,sahu2026adapt,bernardi2026performance,zheng2026analysis}, while recent analyses also investigate the convergence of Adam under heavy-tailed stochastic gradients \cite{pang2026convergence}. Exponential second-moment memory can also generate unfavorable effective learning-rate dynamics even in stationary problems \cite{ReddiKK18}, underscoring the importance of understanding such memory effects. Adaptive methods under changing comparators have also been studied through dynamic-regret formulations \cite{nazari2024dynamic}.

Despite these advances, a theory linking Shampoo's matrix geometry to nonstationary tracking remains largely undeveloped. Existing analyses do not jointly quantify the effects of first-moment memory, evolving two-sided preconditioning, and ridge stabilization on the tracking error. Our work fills this gap, revealing how momentum and matrix geometry shape adaptation
under drift and stochastic noise, and establishing asymptotic minimax-optimal guarantees.

\section{Preliminaries}

Throughout the paper, lowercase letters denote scalars and vectors, uppercase letters matrices, and calligraphic letters sets and classes. For $x,y\in\mathbb R^d$, $\|x\|$ and $\langle x,y\rangle$ denote the Euclidean norm and inner product. For matrices $A,B$ of the same size,
\(
\langle A,B\rangle_F:=\operatorname{tr}(A^{\mathsf T}B),
\,
\|A\|_F:=\sqrt{\langle A,A\rangle_F},
\)
and $\|A\|_{\mathrm{op}}$ denotes the operator norm. For symmetric matrices, $A\succeq B$ denotes the Loewner order and $A\succ0$ positive definiteness; $\lambda_{\min}(A)$ and $\lambda_{\max}(A)$ denote the smallest and largest eigenvalues, respectively. Denote \([x]_+:=\max\{x,0\}\), and for $x>0$, $\log_+x:=\max\{\log x,0\}$, $\mathbf1\{A\}$ denotes the indicator of an event or condition $A$. For a real-valued random variable $X$ and $1\le p<\infty$,
\(
\|X\|_{L^p}:=(\mathbb E|X|^p)^{1/p},
\)
with Euclidean and Frobenius norms replacing $|X|$ for vector- and matrix-valued variables, respectively.

We write $\mathbb E_P,\mathbb P_P$ for expectation and probability under a law $P$, and $\mathbb E_v$ for expectation under a law indexed by hypothesis $v$. For probability laws $P,Q$, $P\otimes Q$ denotes their product law.
If $A\succ0$, we write
\(
\|x\|_A:=\sqrt{x^{\mathsf T}Ax},
\,
\Pi_\Theta^A(z):=\arg\min_{y\in\Theta}\frac12\|y-z\|_A^2.
\)
We use $\Pi_\Theta$ for Euclidean projection and similarly  $\Pi_{\mathcal W}^{F}$ for Frobenius projection onto a matrix set. We write $\operatorname{vec}(W)$ for columnwise vectorization and $\otimes$ for the Kronecker product, so that
\(
\operatorname{vec}(AXB)
=
(B^{\mathsf T}\otimes A)\operatorname{vec}(X).
\)
We write \(I_d\) for the \(d\times d\) identity matrix. For positive sequences \(a_m,b_m\), \(a_m=\mathcal O(b_m)\) means \(a_m\le Cb_m\) for some constant \(C>0\), while \(a_m\lesssim b_m\) (resp. \(\gtrsim,\asymp\)) denotes the corresponding one-sided (resp. reverse, two-sided) bound up to constants independent of \(m\) and the problem parameters. We use \(o_T(1)\) for any quantity converging to zero as \(T\to\infty\). 

We work on a filtered probability space $(\Omega,\mathcal F,(\mathcal F_t)_{t\ge0},\mathbb P)$ and write $\mathbb E[\cdot]$ and $\mathbb P(\cdot)$ for expectation and probability. For a $\sigma$-algebra $\mathcal G$ and $\alpha\ge1$, the conditional $\Psi_\alpha$--Orlicz norm of a real-valued random variable $X$ is
\begin{equation}
\|X\|_{\Psi_\alpha\mid\mathcal G}
:=
\inf\left\{
\begin{aligned}
u>0:\;&u\ \text{is }\mathcal G\text{-measurable and}\,\,\mathbb E\!\left[
\exp\!\left(\frac{|X|^\alpha}{u^\alpha}\right)
\middle|\mathcal G
\right]\le2
\end{aligned}
\right\}.
\end{equation}
For vector and matrix-valued random variables, we apply this definition to Euclidean and Frobenius norms, respectively. Finally, $D_{\mathrm{KL}}(P\|Q)$ denotes the Kullback--Leibler divergence between probability laws $P$ and $Q$.

\subsection{Problem Setup}

Recall the nonstationary stochastic optimization problem
\eqref{eq:nonstationary-convex-opt}. Since Matrix--Shampoo acts on
matrix-valued parameter blocks, we work with
\(W\in\mathcal W\subset\mathbb R^{m\times n}\), where
\(\mathcal W\) is nonempty, compact, and convex. Let \(d=mn\) and define
\(
\Theta:=\{\operatorname{vec}(W):W\in\mathcal W\}\subset\mathbb R^d,
\,
\theta:=\operatorname{vec}(W).
\)
We further write
\(
D:=\sup_{W,V\in\mathcal W}\|W-V\|_F<\infty
\)
for the Frobenius diameter of \(\mathcal W\). By the isometry of
vectorization, \(\Theta\) has the same diameter.

Now, we discuss the following stochastic predictability framework, adapted from \cite{sahu2026provable,sahu2026adapt}.

\begin{assumption}[Stochastic predictability framework]
\label{assump:filtered-predictable-mds}
There exists a filtered probability space
\((\Omega,\mathcal F,(\mathcal F_t)_{t\ge0},\mathbb P)\), with
\(\mathcal F_0=\{\emptyset,\Omega\}\), and an adapted data process
\((X_t)_{t\ge1}\) so that $X_t$ is
\(
\mathcal F_t
\) measurable for each \(t\ge1\). Let \(\Pi_t\) denote the regular conditional law of
\(X_t\) given \(\mathcal F_{t-1}\). Thus \(\Pi_t\) is
\(\mathcal F_{t-1}\)-measurable and
\(
\Pi_t(A)
=
\mathbb P(X_t\in A\mid\mathcal F_{t-1})
\,\text{a.s.}
\)
for every measurable set \(A\). We define the conditional population objective
\begin{equation}
G_t(W)
:=
\mathbb E\!\left[g(W,X_t)\mid\mathcal F_{t-1}\right]
=
\mathbb E_{X\sim\Pi_t}[g(W,X)],
\label{eq:conditional-risk}
\end{equation}
and let
\(
W_t^\star\in\arg\min_{W\in\mathcal W}G_t(W)
\)
denote a measurable population minimizer such that \(\theta_t^\star:=\operatorname{vec}(W_t^\star)\in\operatorname{int}(\Theta)\) lying in the interior of \(\Theta\), where \(\Theta\subset\mathbb R^d\) is compact and convex. Next, assume that, for every \(t\ge1\),

\begin{enumerate}
    \item \textbf{Predictable minimizer.} \(W_t^\star\) is \(\mathcal F_{t-1}\)-measurable.

    \item \textbf{Algorithm adaptedness.}
    The iterate \(W_t\) is \(\mathcal F_{t-1}\)-measurable; it is 
    chosen before observing \(X_t\).

    \item \textbf{Martingale-difference gradient noise.}
    We define the conditional mean gradient map,
    \(
    \mathsf H_t(W)
    :=
    \mathbb E\!\left[
    \nabla_W g(W,X_t)\mid\mathcal F_{t-1}
    \right],
    \)
    and the observed stochastic gradient
    \[
    \mathsf G_t
    :=
    \nabla_W g(W_t,X_t)
    =
    \mathsf H_t(W_t)+\Xi_t.
    \]
    Then the associated noise term \(\Xi_t\) is \(\mathcal F_t\)-measurable and satisfies
    \(
    \mathbb E[\Xi_t\mid\mathcal F_{t-1}]
    =
    0_{m\times n}
    \,\text{a.s.} 
    \) for all $t$.
\end{enumerate}
\end{assumption}
For later analysis, we use the vectorized
quantities
\(
\theta_t:=\operatorname{vec}(W_t),\,
\theta_t^\star:=\operatorname{vec}(W_t^\star),\,
g_t:=\operatorname{vec}(\mathsf G_t),\,
\xi_t:=\operatorname{vec}(\Xi_t),
\)
and define
\(
h_t(\operatorname{vec}(W))
:=
\operatorname{vec}(\mathsf H_t(W)).
\)
The stochastic oracle can therefore be written equivalently as in \cite{sahu2026adapt},
\begin{equation}
g_t=h_t(\theta_t)+\xi_t,
\qquad
\mathbb E[\xi_t\mid\mathcal F_{t-1}]=0.
\label{eq:vector-oracle}
\end{equation}

The conditional distribution $\Pi_t$ allows the data-generating mechanism to evolve with the observed history, a standard formulation in nonstationary stochastic optimization
\cite{sahu2026provable,NEURIPS2021_62e7f2e0,sahu2026adapt}. It covers settings such as policy optimization and reinforcement learning
\cite{kakade2002approximately,schulman2015trust}, contextual bandits and online recommendation
\cite{li2010contextual}, continual learning \cite{parisi2019continual}, and federated learning with time-varying client distributions
\cite{kairouz2021advances}.

\begin{algorithm}[h]
\caption{Matrix--Shampoo}
\label{alg:shampoo}
\begin{algorithmic}[1]
\Require $W_1\in\mathcal W$, horizon $T\ge1$, effective step size
$\alpha>0$, $\beta_1,\beta_2\in[0,1)$, and $\varepsilon>0$, set $\eta=\sqrt{\varepsilon}\,\alpha$,
\State Initialize
$\mathsf M_0=0_{m\times n}$,
$A_0^L=0_{m\times m}$,
$A_0^R=0_{n\times n}$
\For{$t=1,\ldots,T$}
    \State $\mathsf G_t\gets\nabla_W g(W_t,X_t)$

    \State
    \(
    \begin{aligned}[t]
    \mathsf M_t
    &\gets
    \beta_1\mathsf M_{t-1}
    +(1-\beta_1)\mathsf G_t,\,\,\,\,
    \widehat{\mathsf M}_t
    &\gets
    \frac{\mathsf M_t}{1-\beta_1^t}
    \end{aligned}
    \)

    \State
    \(
    \begin{aligned}[t]
    A_t^L
    &\gets
    \beta_2A_{t-1}^L
    +(1-\beta_2)\mathsf G_t\mathsf G_t^{\mathsf T},\\
    A_t^R
    &\gets
    \beta_2A_{t-1}^R
    +(1-\beta_2)\mathsf G_t^{\mathsf T}\mathsf G_t
    \end{aligned}
    \)

    \State
    \(
    \begin{aligned}[t]
    D_t^L
    &\gets
    \frac{A_t^L}{1-\beta_2^t},
    &
    L_t
    &\gets
    \varepsilon I_m+D_t^L,\\
    D_t^R
    &\gets
    \frac{A_t^R}{1-\beta_2^t},
    &
    R_t
    &\gets
    \varepsilon I_n+D_t^R
    \end{aligned}
    \)

    \State
    \(
    \begin{aligned}[t]
    W_{t+1}
    \gets
    \Pi_{\mathcal W}^{F}\!\left(
    W_t
    -\eta
    L_t^{-1/4}
    \widehat{\mathsf M}_t
    R_t^{-1/4}
    \right)
    \end{aligned}
    \)
\EndFor
\end{algorithmic}
\end{algorithm}

The algorithm specified above is the bias-corrected EMA variant of Matrix--Shampoo \cite{shi2023distributed}. Let
$\beta_1,\beta_2\in[0,1)$ denote the first and second-moment parameters,
$\varepsilon, \alpha > 0$ be the ridge parameter and effective step size respectively. For the analysis, we further define
\(
\nu_t:=\operatorname{vec}(\widehat{\mathsf M}_t),
\,
 Q_t(\varepsilon)
:=
\sqrt{\varepsilon}\,
\bigl(R_t^{-1/4}\otimes L_t^{-1/4}\bigr).
\)
Since
$\operatorname{vec}(AXB)=(B^{\mathsf T}\otimes A)\operatorname{vec}(X)$,
the matrix update is equivalently
\begin{equation}
\theta_{t+1}
=
\Pi_\Theta\!\left(
\theta_t-\alpha  Q_t(\varepsilon)\nu_t
\right).
\label{eq:vector-shampoo-update}
\end{equation}
We use $Q_t$ and $Q_t(\varepsilon)$ interchangeably. Since both $\nu_t$ and $Q_t$ depend on the current stochastic gradient
$\mathsf G_t$, they are $\mathcal F_t$-measurable.

Our tracking analysis is formulated directly in terms of the conditional mean-gradient map $\mathsf H_t$, under the standard Lipschitz-gradient condition commonly used in stochastic optimization \cite{ReddiKK18,defossez2022simple,li2023convergence}.

\begin{assumption}[Lipschitz continuity]
\label{assumption:operator}
There exist constants $0\le L<\infty$ such that, almost surely, for
every $t$ and all $W,V\in\mathcal W$,
\begin{align}
\|\mathsf H_t(W)-\mathsf H_t(V)\|_F
&\le
L\|W-V\|_F,
\label{eq:lipschitz-H}
\end{align}
\end{assumption}

Equivalently, under vectorization, \(
\|h_t(x)-h_t(y)\|
\le L\|x-y\|.\)
In particular, since $\mathcal W$ has diameter $D$, under mild regularity condition, $h_t(\theta_t^\star)=0$, and
\begin{equation}
\sup_{W\in\mathcal W}\|\mathsf H_t(W)\|_F
=
\sup_{\theta\in\Theta}\|h_t(\theta)\|
\le LD.
\label{eq:mean-gradient-bound}
\end{equation}
To obtain finite-horizon high-probability guarantees, we impose a
conditional light-tail condition on the stochastic gradient noise \cite{sahu2026provable,sahu2026adapt, pmlr-v99-harvey19a, NEURIPS2021_62e7f2e0}.

\begin{assumption}[Conditional Frobenius-norm sub-Gaussian noise]
\label{assump:cond-subgauss-noise}
There exists $\sigma>0$ such that, for every $t$,
\begin{equation}
\mathbb E\!\left[
\exp\!\left(
\frac{\|\Xi_t\|_F^2}{\sigma^2}
\right)
\middle|
\mathcal F_{t-1}
\right]
\le 2
\,\,\,\,\text{a.s.}
\label{eq:conditional-subgaussian}
\end{equation}
Equivalently,
\(
\mathbb E\!\left[
\exp\!\left(\frac{\|\xi_t\|^2}{\sigma^2}\right)
\middle|\mathcal F_{t-1}
\right]
\le 2
\,\,\text{a.s.}
\)
\end{assumption}

Assumption~\ref{assump:cond-subgauss-noise} permits temporally dependent gradient noise while providing the concentration needed for finite-horizon high-probability bounds. If instead only the one-dimensional projections of \(\xi_t\) are assumed to be \(\sigma\)-sub-Gaussian, then \(\mathbb E\|\xi_t\|^2\lesssim d\sigma^2\). Notably, our tracking guarantees accommodate unbounded stochastic
gradients and do not require the uniform bounded-gradient assumption
imposed in \cite{defazio2101adaptivity}.

\section{Main Results}

We first establish a finite-horizon bound for tracking the moving population optimizer. In addition to Assumptions~\ref{assumption:operator} and \ref{assump:cond-subgauss-noise}, we impose ordinary strong monotonicity directly on the conditional mean-gradient map \cite{sahu2026provable}.

\begin{assumption}[Strong monotonicity]
\label{assump:strong-monotonicity}
There exists $0 < \mu \leq L$ such that, almost surely, for every $t$,
\begin{equation}
\left\langle
\mathsf H_t(W)-\mathsf H_t(V),\,W-V
\right\rangle_F
\ge
\mu\|W-V\|_F^2,
\label{eq:strong-monotonicity-H}
\end{equation}
\end{assumption}
for  all $W,V\in\mathcal W$. We measure nonstationarity through the squared path variation
\begin{equation}
\sum_{t=1}^{T}
\|W_{t+1}^\star-W_t^\star\|_F^2
=
\sum_{t=1}^{T}
\|\theta_{t+1}^\star-\theta_t^\star\|^2
\le \mathcal D_T ,
\label{eq:squared-path-variation}
\end{equation}
without imposing a pointwise restriction on the individual increments. Our analysis does not introduce a predictable proxy preconditioner or impose adaptive strong monotonicity, unlike \cite{sahu2026adapt}. We control the dependence between the current stochastic gradient, momentum, and the preconditioner at the level of the accumulated tracking error. This leads naturally to a finite-horizon bound on the average tracking error.

We next establish a high-probability bound on the average tracking
error of the iterates \(\{W_t\}\) generated by
Algorithm~\ref{alg:shampoo}. 
\(c_{\beta_1}:=\beta_1/(1-\beta_1)\), and \(A_{\beta_1}:=1/(1-\beta_1)\).
Assume that, almost surely for every \(t\),
\(
\lambda_{\min}(D_t^L)\ge\underline\lambda_L,
\,
\lambda_{\min}(D_t^R)\ge\underline\lambda_R,
\)
for some \(\underline\lambda_L,\underline\lambda_R\ge0\),
and define
\begin{align}\label{eq:kappa}
\kappa_\varepsilon
&:=
\frac{\varepsilon}
{\sqrt{(\varepsilon+\underline\lambda_L)
       (\varepsilon+\underline\lambda_R)}},\qquad
K_{\beta_1,\varepsilon}:=4c_{\beta_1}\sqrt{\kappa_\varepsilon}
   +2A_{\beta_1}+2\kappa_\varepsilon.
\end{align}

\begin{theorem}[High-probability average tracking]
\label{thm:main-tracking}
Suppose Assumptions~\ref{assump:filtered-predictable-mds}--
\ref{assump:strong-monotonicity} hold, and suppose that the optimizer path satisfies \eqref{eq:squared-path-variation}.
Fix $\delta\in(0,1)$ and let $\varepsilon>0$ as in Algorithm~\ref{alg:shampoo}. If
\(0<\alpha<
\frac{\mu(1-\beta_1)}{384L^2},\)
the following holds with probability at least $1-\delta$,
\begin{align}
\frac1T\sum_{t=1}^T
\|W_t-W_t^\star\|_F^2
\lesssim \Bigg\{
\underbrace{
\frac{\|W_1-W_1^\star\|_F^2}{\mu\alpha T}
}_{\text{initialization}}
+&
\underbrace{
\frac{\mathcal D_T}{\mu^2\alpha^2T}
}_{\text{optimizer drift}}
\nonumber+
\underbrace{
\frac{\beta_1^2}{1-\beta_1}
\frac{\mathcal D_T}{\mu\alpha T}
}_{\substack{\text{momentum-induced}\\ \text{drift}}}
+
\underbrace{
\frac{\alpha}{\mu}
K_{\beta_1,\varepsilon}  \, \mathfrak Z_{T,\beta_1,\delta}
}_{\substack{\text{filtered stochastic}\\ \text{noise}}}
\nonumber+
\underbrace{
\frac{\sigma^2}{\mu^2T}
\log\frac8\delta
}_{\substack{\text{stochastic}\\ \text{fluctuation}}}\\
&+
\underbrace{
\frac{\beta_1}{1-\beta_1}
\frac{D\mathfrak B_\delta}{\mu T}
}_{\substack{\text{momentum boundary} \\ \text{term}}}
+
\underbrace{
\frac{\mathfrak G_\delta}{\mu^2\varepsilon^2}
}_{\text{preconditioner perturbation}}
\Bigg\}
\label{eq:main-tracking-bound}
\end{align}
\(
\mathfrak Z_{T,\beta_1,\delta} \asymp
\sigma^2\left[
\frac{1-\beta_1}{1+\beta_1}\,
+
\frac{\log(4/\delta)}{T(1-\beta_1^2)}
+
\frac{\beta_1^2\log(4/\delta)}
{T(1-\beta_1^2)^2} + o_T(1)
\right], 
\mathfrak B_\delta \asymp
LD+\sigma\sqrt{\log(4/\delta)},
\,\,
\mathfrak G_\delta
\asymp
(LD)^6+\sigma^6\log^3(4/\delta).\\
\label{eq:BG-main}
\)
The exact definitions of these terms are in \hyperref[sec:app]{Appendix}.
\end{theorem}

Theorem~\ref{thm:main-tracking} separates the two principal effects of first-moment memory, controlled by \(\beta_1\), and the evolving matrix preconditioner, controlled through the ridge parameter \(\varepsilon\). The initialization term measures sensitivity to the starting point, while the direct drift term captures the difficulty of following a moving optimizer. First-moment memory contributes an additional drift penalty through
stale gradient information, while temporal averaging reduces the
leading stochastic-noise contribution. The filtered noise term
\(K_{\beta_1,\varepsilon}\mathfrak Z_{T,\beta_1,\delta}\)
makes its joint dependence on first-moment memory and ridge
stabilization explicit. The  preconditioner term $\frac{\mathfrak G_\delta}{\mu^2\varepsilon^2}$ quantifies the perturbation relative to identity, induced by the evolving matrix geometry.

The bound in \eqref{eq:main-tracking-bound} makes the drift--noise tradeoff explicit while retaining its dependence on matrix geometry. The \(\beta_1\)-dependent momentum-induced drift term
\(
\frac{\mathcal D_T}{\mu\alpha T}
\frac{\beta_1^2}{1-\beta_1}
\)
increases with \(\beta_1\), reflecting greater sensitivity to optimizer drift. For fixed \(\beta_1<1\) and
\(\delta\), the filtered-noise contribution satisfies, as
\(T\to\infty\),
\[
\begin{aligned}
K_{\beta_1,\varepsilon}
\mathfrak Z_{T,\beta_1,\delta}
\asymp \sigma^2
\Bigg[
&4\sqrt{\kappa_\varepsilon}-2\kappa_\varepsilon+\frac{2(1-2\sqrt{\kappa_\varepsilon}
                +2\kappa_\varepsilon)}
       {1+\beta_1}
\Bigg],
\end{aligned}
\]
its leading coefficient decreases with \(\beta_1\), even though
\(K_{\beta_1,\varepsilon}\) itself increases, reflecting stochastic noise suppression through temporal averaging. Thus, apart from the
\(T^{-1}\) momentum-boundary and finite-horizon noise corrections,
the leading bound favors longer memory when noise dominates and
shorter memory through faster adaptation when optimizer drift dominates. 

Theorem~\ref{thm:main-tracking} exposes the stability--geometry
tradeoff induced by ridge regularization. The filtered-noise
coefficient \(K_{\beta_1,\varepsilon}\) is nondecreasing in
\(\kappa_\varepsilon\), and hence in \(\varepsilon\).
A smaller ridge preserves stronger matrix preconditioning
and can tighten the filtered-noise bound, but enlarges the
\(\varepsilon^{-2}\) perturbation term. Conversely, a larger
ridge reduces this perturbation while driving
\(Q_t(\varepsilon)\to I\), progressively removing Shampoo's
directional rescaling and making the update approach
momentum-SGD at fixed \(\alpha\).
These stability concerns accord with empirical reports of
adaptive-optimizer instability \cite{bai2025adaptive}
and degraded adaptation under distribution shifts
\cite{lyle2023understanding,lee2024slow}.

To formalize minimax optimality, let
\(
\mathcal M_T
=
\mathcal M_T(\mu,L,\sigma,\mathcal D_T,\mathcal W)
\)
denote the class of all filtered nonstationary stochastic optimization problems satisfying Assumptions~\ref{assump:filtered-predictable-mds}--\ref{assump:strong-monotonicity} and \eqref{eq:squared-path-variation}
For \(\delta\in(0,1)\), define the high-probability minimax tracking radius
\begin{align}\label{eq:def_rtdelta}
\mathfrak R_{T,\delta}
:=\inf_{\mathcal A}\inf\Bigg\{q\ge0:
\sup_{P\in\mathcal M_T}\mathbb P_P\! &\left(
\frac1T\sum_{t=1}^T
\|W_t^{\mathcal A}-W_t^\star\|_F^2>q
\right) \le\delta\Bigg\}
\end{align}
where the outer infimum is taken over all possibly randomized causal optimization procedures \(\mathcal A\).

\begin{theorem}[Asymptotic Minimax optimality under squared path variation]
\label{thm:minimax}
Suppose there exists \(W^0\in\mathcal W\), \(U\in\mathbb R^{m\times n}\) with \(\|U\|_F=1\), and \(R>0\) such that
\begin{align}\label{eq:line_condition}
\{W^0+zU:\ |z|\le R\}\subseteq\mathcal W .
\end{align}
For $T^{-1/2}\lesssim \mathcal D_T\lesssim T$, define
\(
J_\star
:=
\left(
\frac{\mu^2\mathcal D_TT^2}{\sigma^2}
\right)^{1/3},
\,
r_T^\star
:=
\frac{\sigma^2J_\star}{\mu^2T}
=
\frac{\sigma^{4/3}\mathcal D_T^{1/3}}
{\mu^{4/3}T^{1/3}} .
\)
Then there is a constant $c_1>0$ such that, every possibly randomized causal optimization procedure \(\mathcal A\) satisfies
\[
\sup_{P\in\mathcal M_T}
\mathbb P_P\!\left\{
\frac1T\sum_{t=1}^T
\|W_t^{\mathcal A}-W_t^\star\|_F^2
\ge c_1\,r_T^\star
\right\}
\ge \frac12 
\]
 for all sufficiently large  $T > T_0$. The above lower bound combined with \eqref{eq:def_rtdelta} implies that, for every \(\delta<1/2\), we have 
\(
\mathfrak R_{T,\delta}\ge c_1\,r_T^\star .
\)
Conversely, fix \(\delta\in(0,1/2)\) and \(\bar\beta_1<1\), and suppose \(0\le\beta_1\le\bar\beta_1\), then there exists tuning choices
\((\alpha_\star,\varepsilon_\star)\), common to all
\(P\in\mathcal M_T\), and a constant \(c_2>0\) such that for all $T > T_0$ Matrix--Shampoo satisfies
\[
\sup_{P\in\mathcal M_T}
\mathbb P_P\!\left\{
\frac1T\sum_{t=1}^T
\|W_t^{\text{Shampoo}}-W_t^\star\|_F^2
\geq
c_2\,r_T^\star
\right\}
\le\delta .
\]
Consequently,
\(
\mathfrak R_{T,\delta}
\asymp
\,r_T^\star.
\)
The constants \(c_1, c_2\) may depend on the fixed problem parameters. 
If, additionally, \(T^{-1/2}=o(\mathcal D_T)\) and
\(\mathcal D_T=o(T)\), these constants can be
chosen independent of all problem parameters, for large enough $T$, obtaining the minimax optimal rate.
\end{theorem}

The line-segment condition in \eqref{eq:line_condition} follows from the compact convex geometry of \(\mathcal W\) together with the existence of an interior optimizer. Theorem~\ref{thm:minimax} shows that the rate
\(
r_T^\star
=
\frac{\sigma^{4/3}\mathcal D_T^{1/3}}
{\mu^{4/3}T^{1/3}}
\)
is minimax sharp in its joint dependence on the horizon \(T\), drift budget \(\mathcal D_T\), noise level \(\sigma\), and the strong-monotonicity parameter \(\mu\). In particular, larger noise and greater optimizer variation fundamentally increase the difficulty of tracking, whereas stronger curvature improves trackability. Since appropriately tuned Matrix--Shampoo attains the matching upper bound, no causal optimization procedure can improve this worst-case order uniformly over \(\mathcal M_T\), up to universal constants. Under sub-linear variation, the tuning used to certify the worst-case rate has diverging ridge factor, driving Shampoo toward momentum-SGD, consistent with the large-damping behavior \cite{savarese2021domain,silva2024finetuningadaptivestochasticoptimizers}.

We next characterize burn-in under constant stepsize relative to
\(\overline{\mathcal R}_\delta\) (defined in Theorem~\ref{thm:burnin}), the persistent contribution to the tracking
bound determined by optimizer drift, stochastic noise,
first-moment memory, and preconditioner perturbations.
Burn-in measures the time required for the initialization and
momentum transients to decay to this scale.

\begin{theorem}[High-probability burn-in time for Matrix--Shampoo]
\label{thm:burnin}
Suppose Assumptions~\ref{assump:filtered-predictable-mds}--
\ref{assump:strong-monotonicity} hold, in addition,
\(
\|W_{t+1}^\star-W_t^\star\|_F\leq\Delta,
\) for each $t\geq1$.
Let
\[
0<\alpha\leq\frac{\mu}{8L^2},
\qquad
\rho:=1-\frac{3\alpha\mu}{16},
\]
and let \(\overline{\mathcal R}_\delta\) denote the  high-probability post-transient tracking radius
\[
\begin{aligned}
\overline{\mathcal R}_\delta
:={}&
\frac{1}{\delta}
\Bigg[
\frac{\beta_1(1+\beta_1)}{(1-\beta_1)^2}
\left(1+\frac{L^2}{\mu^2}\right)
\left\{
\alpha^2\kappa_{\varepsilon}\bigl((LD)^2+\sigma^2\bigr)+\Delta^2
\right\}
+
\kappa_{\varepsilon} \frac{\sigma^2}{\mu^2}
\frac{1-\beta_1}{1+\beta_1}
+
\frac{(LD)^6+\sigma^6}{\mu^2\varepsilon^2}
\Bigg] + 
\frac{\Delta^2}{\alpha^2\mu^2},
\end{aligned}
\]
$q = \max\{\rho,\beta_1\}$. Then, for every fixed \(T\geq T_{\mathrm{burn}}\) and $\delta \in (0,1)$,
\(
\mathbb P\!\left(
\|W_T-W_T^\star\|_F^2
\lesssim
\overline{\mathcal R}_\delta
\right)
\geq1-\delta,
\)
where
\[
T_{\mathrm{burn}}(\delta)
= 1 \,+ \,
\mathcal{O}\!\left(
\max\left\{
\begin{aligned}
&\frac{1}{1-\rho}
\log_+\!\left(
\frac{\|W_1-W_1^\star\|_F^2}
{\overline{\mathcal R}_\delta}
\right)
\\[0.5em]
&\frac{1}{1-q}
\log_+\!\left(
\frac{\alpha \kappa_{\varepsilon}\sigma^2}
{\mu\,\delta \,(1-q)\,\overline{\mathcal R}_\delta}
\right)
\end{aligned}
\right\}
\right).
\]
\end{theorem}
Both the high-probability burn-in time and tracking radius remain \(O(1)\) in \(T\), as does the iterative bound \eqref{eq:tracking-error} in the \hyperref[sec:app]{Appendix}. For comparison, existing adaptive-method guarantees may retain horizon-dependent tracking floors; Adam's high-probability bound holds simultaneously over a finite horizon and contains an \(O(\log^2 T)\) contribution \cite{sahu2026adapt}.

\section{Nonconvex projected stationarity}\label{sec:projected-stationarity}

The preceding analysis establishes tracking guarantees under strong monotonicity. We now drop this condition and instead study first-order stationarity for general smooth time-varying objectives. We write
\(F_t(\operatorname{vec}(W)):=G_t(W)\), with
\(\nabla F_t=h_t\), and define the Euclidean projected-gradient
mapping
\begin{align}\label{eq:def_grad}
\mathcal G_{\alpha,t}(\theta)
:=
\frac{1}{\alpha}
\left[
\theta-\Pi_{\Theta}
\bigl(\theta-\alpha\nabla F_t(\theta)\bigr)
\right].
\end{align}
This is the natural stationarity measure for the Euclidean-projected Matrix--Shampoo. In particular,
\(\mathcal G_{\alpha,t}(\theta)=0\) if and only if
\(-\nabla F_t(\theta)\) belongs to the normal cone of \(\Theta\) at \(\theta\).
Unlike the tracking analysis, the result below does not involve
the optimizer trajectory or require a path-variation budget. However, we impose the standard lower-boundedness condition used in nonconvex stochastic optimization  \cite{ReddiKK18, defossez2022simple, li2023convergence}.

\begin{assumption}[Uniform lower boundedness]
\label{assump:lower-bounded}
There exists a constant \(F_\star>-\infty\) such that
\(
F_t(\theta)\geq F_\star
\)
for every \(t\geq1\) and every \(\theta\in\Theta\).
\end{assumption}
To quantify nonstationarity, define the pathwise objective variation
\(\label{eq:def_V_T}
\mathcal V_T
:=
\sum_{t=1}^{T-1}
\left[
F_{t+1}(\theta_{t+1})
-
F_t(\theta_{t+1})
\right]_+ .
\)
This quantity measures the cumulative increase in the objective induced by temporal changes in the optimization problem, evaluated along the realized iterates. We now state the high-probability average projected-stationarity guarantee for Matrix--Shampoo.

\begin{theorem}[High-probability projected-gradient rate]
\label{thm:nonconvex-shampoo-hp}
Suppose Assumptions~\ref{assump:filtered-predictable-mds}--\ref{assump:cond-subgauss-noise}, \ref{assump:lower-bounded} hold. Let
\(\theta_1,\ldots,\theta_{T+1}\) be generated by the Euclidean-projected
Matrix--Shampoo update. If
\(
0<\alpha\leq\frac{1}{2L},
\)
then, for every \(T\geq1\) and every \(\delta\in(0,1)\), with probability at least \(1-\delta\),
\begin{align}\label{thm:thm_4}
\frac1T\sum_{t=1}^T
\left\|
\mathcal G_{\alpha,t}(\theta_t)
\right\|^2
\lesssim\,
\frac{
\bigl(
F_1(\theta_1)-F_\star+\mathcal V_T
\bigr)
}{
\alpha T
}
+
\kappa_{\varepsilon}L^2D^2\beta_1^2
+
\kappa_{\varepsilon}\mathfrak Z_{T,\beta_1,\delta}
+
\frac{\mathfrak G_\delta}{\varepsilon^2},
\end{align}
where $\mathfrak G_\delta$ and $\mathfrak Z_{T,\beta_1,\delta}$ are as in Theorem~\ref{thm:main-tracking}.
\end{theorem}

Theorem~\ref{thm:nonconvex-shampoo-hp} separates objective variation, momentum memory, filtered noise, and preconditioner perturbations, retaining the explicit \(\beta_1\) and \(\kappa_\varepsilon\) dependence hidden by the generic bound \(\|\mathcal G_{\alpha,t}(\theta_t)\|^2\le L^2D^2\). At fixed \(\alpha\), the objective-variation term vanishes when \(\mathcal V_T=o(T)\). For fixed \(\beta_1<1\) and \(\delta\), the filtered noise contribution, \[\kappa_{\varepsilon} \mathfrak Z_{T,\beta_1,\delta}, \asymp \frac{\varepsilon}
{\sqrt{(\varepsilon+\underline\lambda_L)
       (\varepsilon+\underline\lambda_R)}}\sigma^2(1-\beta_1)/(1+\beta_1),\] whereas the memory penalty \(\kappa_\varepsilon L^2D^2\beta_1^2\) captures discrepancies between past and current mean gradients. Their opposing dependence on \(\beta_1\) makes the momentum tradeoff explicit. Also, larger \(\varepsilon\) drives \(\kappa_\varepsilon\) toward one,
making both the memory and filtered-noise contributions
increasing, while reducing \(\mathfrak G_\delta/\varepsilon^2\)
and weakening Shampoo's directional rescaling. Thus, the qualitative stability--geometry tradeoff from Theorem~\ref{thm:main-tracking} persists.

\section{Experiments}
\label{sec:numerics}

We perform synthetic nonstationary experiments to examine the main phenomena established in theory. Across experiments, the target evolves over time according to,
\(
W_{t+1}^\star
=
W_t^\star
+
\Delta_t\frac{U_t}{\|U_t\|_F},
\,
(U_t)_{ij}\sim N(0,1),
\)
allowing the effects of drift and stochastic-gradient noise to be varied separately.  First, we study the effect of $\beta_1$ on three nonstationary matrix problems: convex  Kronecker quadratic, Huber Regression, and nonconvex phase retrieval.

\begin{figure}[H]
    \centering
    \includegraphics[width=10cm]
    {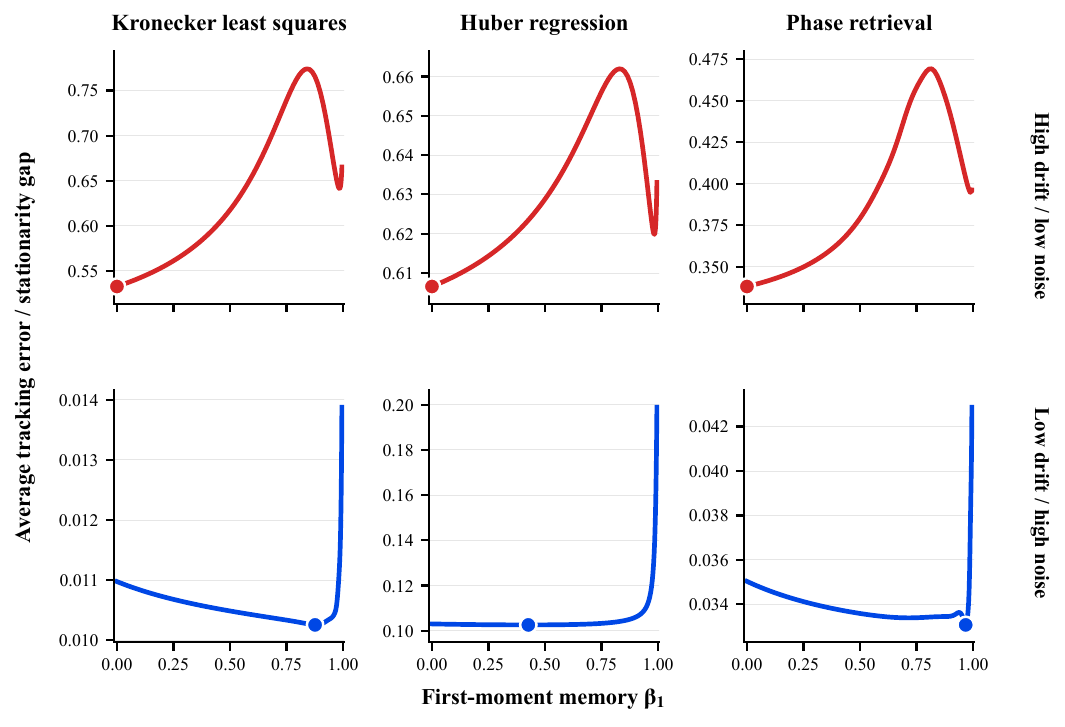}
    \caption{\textbf{Momentum drift--noise tradeoff.} Results for \textbf{Kronecker} (left), \textbf{Huber} (center), and \textbf{phase retrieval} (right). High-drift/Low-noise regime, $(\Delta,\sigma^2)=(0.15,10^{-2})$ (top), Low-drift/High-noise regime, $(\Delta,\sigma^2)=(0.003,1)$ (bottom). Enlarged markers indicate the empirically optimal value of $\beta_1$ over $T = 500$ iterations. Across all three objectives, the optimal $\beta_1$ shifts from zero in the drift-dominated regime towards right in the noise-dominated regime, illustrating that short memory facilitates tracking while longer memory suppresses stochastic noise.}
    \label{fig:beta1-tradeoff}
\end{figure}
We next test the minimax scaling \(r_T^\star\) predicted by
Theorem~\ref{thm:minimax} for appropriately tuned Matrix--Shampoo.
\begin{figure}[htbp]
    \centering
    \includegraphics[width=10cm]
    {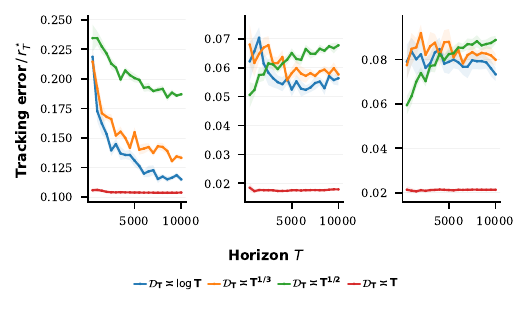}
    \caption{\textbf{Empirical validation of the minimax tracking rate.} Results for the \textbf{Kronecker} (left), \textbf{Huber} (center), and \textbf{soft-label logistic} (right) problems under the path-variation regimes $\mathcal D_T\asymp \log T,\;T^{1/3},\;T^{1/2}$, and $T$.
Across all three problems, the normalized errors remain overall stable over the horizon, supporting the minimax scaling.}
    \label{fig:minimax-scaling}
\end{figure}

Finally, we examine the stability--geometry tradeoff induced
by \(\varepsilon\), comparing tracking error and stationarity
gaps with the realized
preconditioned filtered-noise energy
\(
\mathfrak Z_T^Q(\varepsilon)
:=T^{-1}\sum_{t=1}^T\|Q_t(\varepsilon)z_t\|^2,
\)
where \(z_t\) is the bias-corrected exponential average of
gradient noise. This diagnostic captures how the evolving matrix geometry
acts on stochastic noise along the optimization,
complementing the worst-case coefficient
\(K_{\beta_1,\varepsilon}\), which is nondecreasing in
\(\varepsilon\).

\begin{figure}[htbp]
    \centering
    \includegraphics[width=10cm]
    {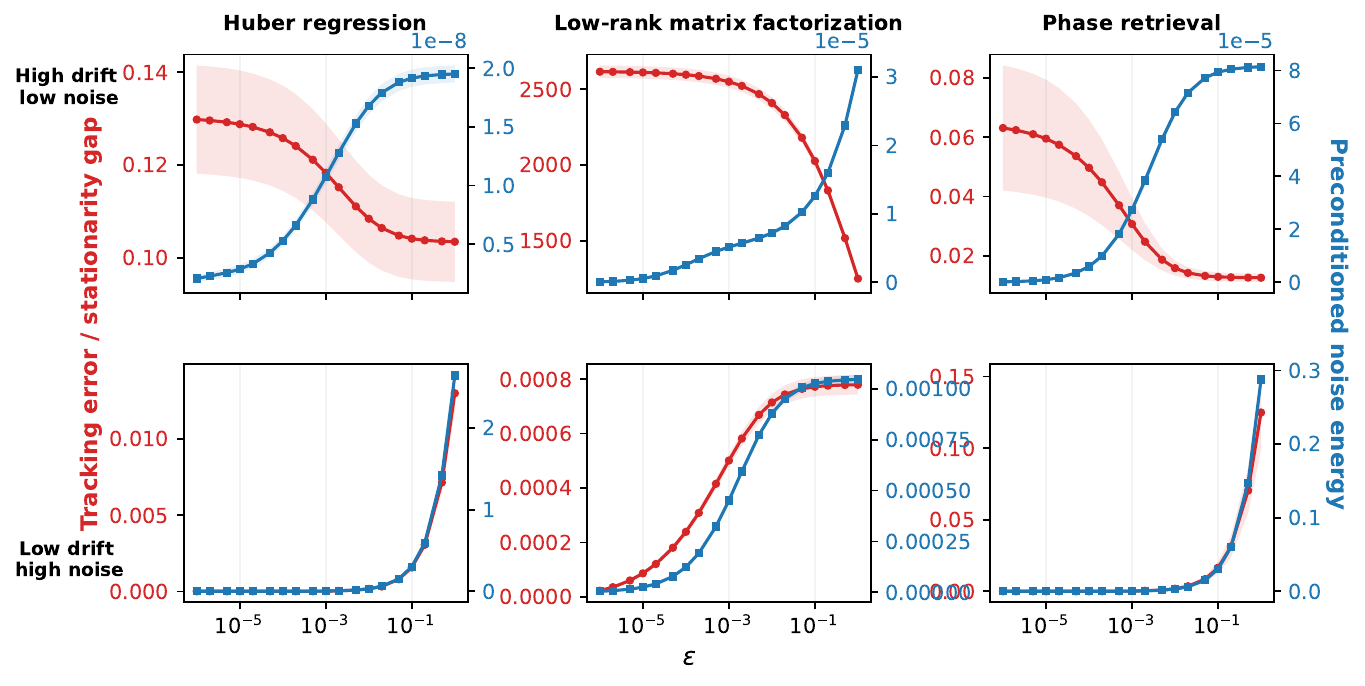}
   \caption{\textbf{Ridge stability--geometry tradeoff.} Average tracking error for convex and average projected-stationarity gap for non-convex objectives. In high-drift, low-noise regime (top), the realized preconditioned filtered-noise energy is small, so the preconditioner-perturbation term  $\mathfrak G_\delta/\varepsilon^2$  dominates, decreasing tracking error as $\varepsilon$ increases. In the other regime,  $\mathfrak Z_T^Q(\varepsilon)$ becomes dominant as $\varepsilon$ increases, exhibiting the complementary behavior.}
    \label{fig:ridge-geometry}
\end{figure}
\begin{figure}[htbp]
    \centering
    \includegraphics[width=10cm]
    {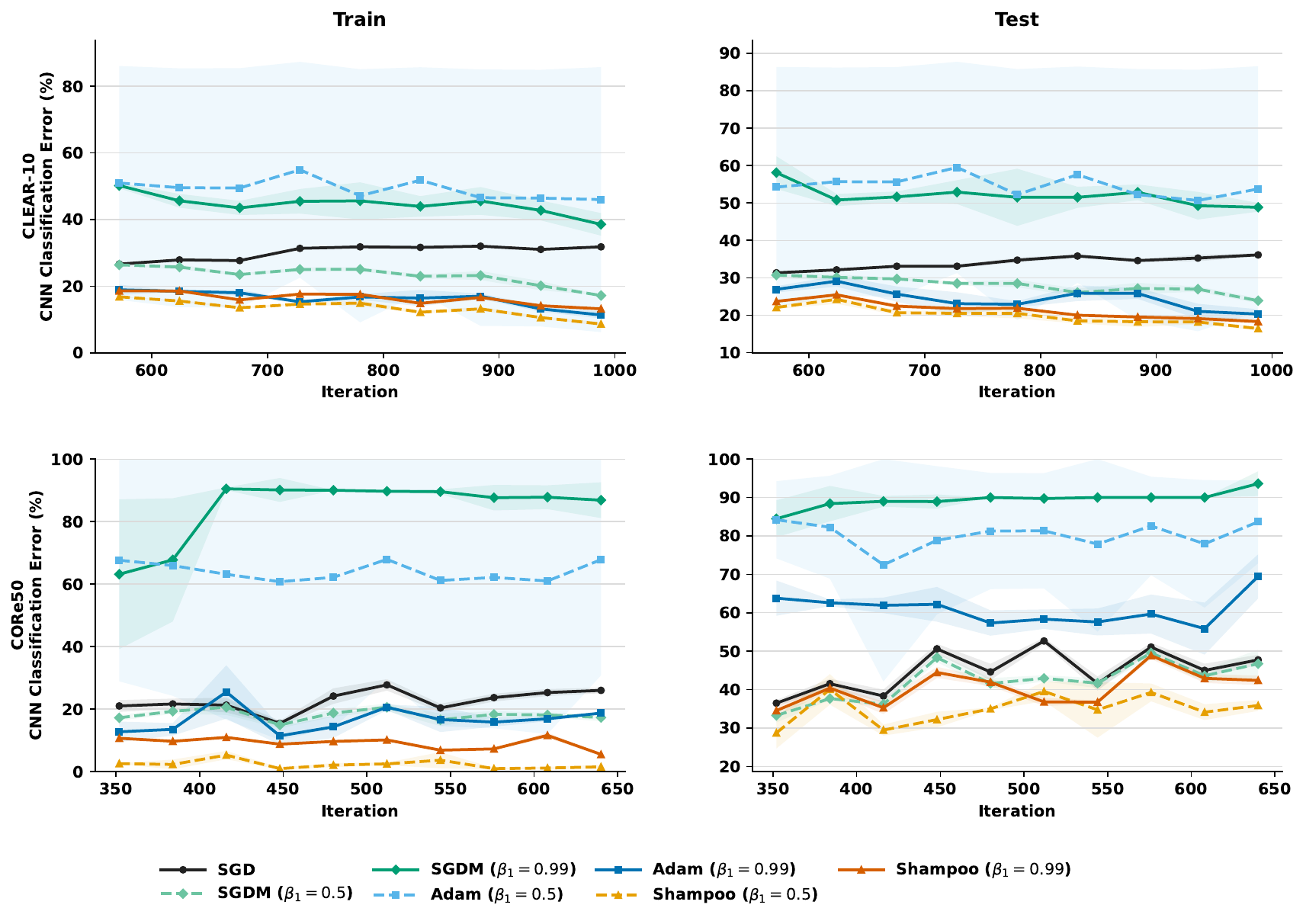}
    \caption{\textbf{CNN performance on CLEAR-10 and CORe50.}
    Mean classification error at bucket-end iterations on the training (left) and test (right) splits for SGD and for momentum-SGD, Adam, and Shampoo with $\beta_1\in\{0.5,0.99\}$. Shampoo performs consistently across the $\beta_1$ values, whereas Adam exhibits substantial instability with shorter memory and SGDM deteriorates under longer memory \cite{sahu2026provable}.}
    \label{fig:Clear10}
\end{figure}

\subsection{Real-Data Experiments}

We use CLEAR-10 and CORe50 to evaluate optimizer adaptation
under natural distribution shifts, with ResNet-18 weight blocks
providing a practical setting for Shampoo's two-sided matrix
preconditioning. CLEAR-10 organizes timestamped YFCC100M imagery
from 2004--2014 into chronological buckets that capture the
evolution of real-world visual concepts
\cite{lin2021clear,10.1145/2812802}.
Sequential training across these buckets produces a nonstationary
learning problem with evolving gradient statistics.
Using the predefined training and test splits, we trained
ResNet-18 end to end with ridge-regularized multiclass
cross-entropy and compared SGD, momentum SGD, Adam, and EMA-Shampoo.

CORe50 complements this temporal evolution with distribution
shifts across 11 acquisition sessions, characterized by changes
in illumination, background, viewpoint, object pose, and
occlusion \cite{pmlr-v78-lomonaco17a}.
The same objects recur across sessions, allowing us to examine
adaptation to changing visual conditions while keeping the
training objects fixed. We retained 50 frames per object per
session and constructed a custom session-ordered, object-disjoint
protocol with 40 training objects and 10 held-out test objects.
We trained ResNet-18 end to end using the same optimizers,
comparing their behavior as the acquisition conditions evolve.

\section{Conclusion}

We developed a finite-horizon theory for Matrix--Shampoo in
nonstationary stochastic optimization. Our high-probability
bounds separate optimizer drift, stochastic noise, first-moment
memory, and evolving matrix geometry, making their tracking
tradeoffs explicit. Under the stated squared path-variation
conditions, we established a matching minimax lower bound and
showed that appropriately tuned Matrix--Shampoo attains the
optimal tracking rate over all causal optimization procedures.

Beyond strong monotonicity, we obtained high-probability
projected-stationarity guarantees for smooth time-varying
nonconvex objectives, revealing analogous momentum and ridge
effects. Synthetic experiments support the predicted tradeoffs
and minimax scaling, while experiments on CLEAR-10 and CORe50
illustrate Shampoo's stability across the tested momentum
settings under natural distribution shifts. Future work should
incorporate stepsize grafting and derive tracking bounds with
explicit dependence on second-moment memory, bringing the
theory closer to practical Shampoo implementations.


\bibliography{references}

\clearpage
\appendix
\thispagestyle{empty}

\onecolumn

\title{
Supplementary Material:
Minimax Optimality of Shampoo for
Nonstationary Stochastic Optimization
}

\section{Appendix}\label{sec:app}
We shall repeatedly use the following notation in the proofs.
Throughout the proofs, we will write $\beta:=\beta_1\in[0,1)$ and define
\[
c_\beta:=\frac{\beta}{1-\beta},
\qquad
A_\beta:=1+c_\beta=\frac1{1-\beta}, \qquad \kappa_\varepsilon
:=
\frac{\varepsilon}
{\sqrt{(\varepsilon+\underline\lambda_L)
       (\varepsilon+\underline\lambda_R)}},  \qquad
K_{\beta,\varepsilon}
:=4c_{\beta}\sqrt{\kappa_\varepsilon}
   +2A_{\beta}+2\kappa_\varepsilon.
\]
We also define
\begin{equation}
 e_t:=\theta_t-\theta_t^\star,
 \qquad
 \mathsf E_T:=\sum_{t=1}^T\norm{e_t}^2,
 \qquad
 \mathsf N_T:=\sum_{t=1}^T\norm{\nu_t}^2.
 \label{eq:energies}
\end{equation}
The vectorized bias-corrected first moment satisfies
\begin{equation}
 \nu_t=a_t\nu_{t-1}+b_tg_t
 =\sum_{s=1}^tw_{t,s}g_s,
 \qquad
 a_t:=\frac{\beta(1-\beta^{t-1})}{1-\beta^t},
 \qquad
 b_t:=\frac{1-\beta}{1-\beta^t},
 \qquad
 w_{t,s}:=\frac{(1-\beta)\beta^{t-s}}{1-\beta^t},
 \label{eq:ab-recursion}
\end{equation}
with $\nu_0=0$. For the summation-by-parts arguments, define
\begin{equation}
 d_t:=\frac{a_t}{b_t}
 =\frac{\beta(1-\beta^{t-1})}{1-\beta}.
 \label{eq:d-definition}
\end{equation}
These coefficients satisfy
\[
d_1=0,
\qquad
0\leq d_t\leq c_\beta,
\qquad
d_{t+1}-d_t=\beta^t.
\]
Finally, separate the bias-corrected first moment into its filtered
mean-field and noise components:
\begin{equation}
 \nu_t=u_t+z_t,
 \qquad
 u_t:=\sum_{s=1}^tw_{t,s}h_s(\theta_s),
 \qquad
 z_t:=\sum_{s=1}^tw_{t,s}\xi_s,
 \qquad
 \mathsf Z_T:=\sum_{t=1}^T\norm{z_t}^2.
 \label{eq:signal-noise-split}
\end{equation}

\subsection{Proof of Theorem~\ref{thm:main-tracking}}

\begin{lemma}[Consequences of conditional norm-sub-Gaussianity]
\label{lem:subg-consequences}
There is a universal numerical constant $C_0>0$ with the following properties.
    \begin{enumerate}[label=(\roman*)]
        \item For every $t$, \(\E[\norm{\xi_t}^2\mid\F_{t-1}]\leq\sigma^2\log2\). Additionally, for every $u\geq0$, \(\Pp\!\left(\norm{\xi_t}>\sigma\sqrt u\mid\F_{t-1}\right) \leq2e^{-u}\).
        \item If $v_t$ is $\F_{t-1}$-measurable, then for every deterministic $a>0$ and $\delta\in(0,1)$, with probability at least $1-\delta$,
        \begin{equation}
         \left|\sum_{t=1}^T\ip{v_t}{\xi_t}\right| \leq a\sum_{t=1}^T\norm{v_t}^2 +\frac{C_0\sigma^2}{a}\log\frac2\delta
         \label{eq:flexible-martingale}.
        \end{equation}
    \end{enumerate}
\end{lemma}

\begin{proof}
Conditional Jensen and Assumption~\ref{assump:cond-subgauss-noise} give 
\(\exp\!\left(\frac{\E[\norm{\xi_t}^2\mid\F_{t-1}]}{\sigma^2}\right) \leq \E\!\left[ e^{\norm{\xi_t}^2/\sigma^2}\mid\F_{t-1} \right]\leq2\). 
Taking logarithms proves the conditional second moment bound. Conditional Markov inequality gives
\(\Pp(\norm{\xi_t}^2>\sigma^2u\mid\F_{t-1}) \leq e^{-u} \E[e^{\norm{\xi_t}^2/\sigma^2}\mid\F_{t-1}] \leq2e^{-u}\),
which proves (i).

Using the layer-cake identity, $|\ip{v}{\xi_t}|\leq\norm v\norm{\xi_t}$, and (i) imply, for every integer $p\geq1$,
\begin{align*}
 \E[|\ip{v}{\xi_t}|^p\mid\F_{t-1}] = p\int_0^\infty x^{p-1}
  \Pp\!\left(|\ip{v}{\xi_t}|>x \mid\F_{t-1} \right )\,dx \leq2p\int_0^\infty x^{p-1}
  \exp\!\left(-\frac{x^2}{\sigma^2\norm v^2}\right)\,dx \le \left(C\sigma\norm v\sqrt p\right)^p
\end{align*}
for a universal $C$ (the assertion is trivial for $v=0$). Put $X:=\ip{v}{\xi_t}$ and $z:=C|\lambda|\sigma\norm v$. Conditional centering removes the linear term from the exponential series, and hence
\[
    \E[e^{\lambda X}\mid\F_{t-1}]\leq1+\sum_{p=2}^{\infty}\frac{(z\sqrt p)^p}{p!}
    \leq \exp(C_0\lambda^2\sigma^2\norm v^2),\qquad \lambda\in\R.
\]
The second inequality is due to Stirling's lower bound $p!\geq(p/e)^p$ which makes the $p$th summand at most $(ez/\sqrt p)^p$. For $0\le z\le(2e)^{-1}$, we get $1+\sum_{p=2}^{\infty}(z\sqrt p)^p/p!\le 1+\sum_{p=2}^{\infty}(ez)^p\le e^{2e^2z^2}$. For \(z>(2e)^{-1}\), split at $N:=\lceil4e^2z^2\rceil$. For $p>N$ the summand is at most $2^{-p}$, while the first $N=O(z^2)$ terms are together at most $e^{C_1z^2}$. The same conditional inequality holds pointwise for a predictable $v_t$. Hence
\[
 \exp\!\left\{
  \lambda\sum_{t=1}^k\ip{v_t}{\xi_t}
  -C_0\lambda^2\sigma^2\sum_{t=1}^k\norm{v_t}^2
 \right\},
 \qquad 0\leq k\leq T,
\]
is a nonnegative supermartingale with initial value one. Markov's inequality at $k=T$, applied to $\lambda$ and $-\lambda$, gives with probability at least $1-\delta$,
\[
 \left|\sum_{t=1}^T\ip{v_t}{\xi_t}\right|
 \leq
 C_0\lambda\sigma^2\sum_{t=1}^T\norm{v_t}^2
 +\frac1\lambda\log\frac2\delta.
\]
Choose $\lambda=a/(C_0\sigma^2)$ and absorb universal numerical factors into $C_0$.
\end{proof}

\begin{lemma}[High-probability sixth-moment average]
\label{lem:sixth-moment}
Let \(\ell_\delta:=\log\dfrac4\delta\). There is a universal numerical constant $C_6\geq 1$ such that, with probability at least $1-\delta/4$, \(\dfrac1T\sum_{t=1}^T\norm{g_t}^6  \leq   \mathfrak G_\delta :=C_6\left\{(LD)^6+\sigma^6\ell_\delta^3\right\}\).
\end{lemma}

\begin{proof}
Fix $q\geq1$. By Lemma~\ref{lem:subg-consequences}(i), the layer-cake identity gives, conditionally on $\F_{t-1}$,
\begin{equation}
     \E[\norm{\xi_t}^q\mid\F_{t-1}] = q\int_0^\infty x^{q-1} \Pp(\norm{\xi_t}>x\mid\F_{t-1})\,dx  
 \leq2q\int_0^\infty x^{q-1}\exp\!\left(-\frac{x^2}{\sigma^2}\right)\,dx =q\sigma^q\Gamma(q/2).
 \label{eq:q-moment}
\end{equation}
The standard Gamma-function bound $(q\Gamma(q/2))^{1/q}\leq C\sqrt q$ therefore implies $\left(\E\norm{\xi_t}^q\right)^{1/q}\leq C\sigma\sqrt q$. The deterministic mean-field bound \eqref{eq:mean-gradient-bound} and Minkowski's inequality yield \(\left(\E\norm{g_t}^q\right)^{1/q}\leq LD+C\sigma\sqrt q\). Now let $p:=\ell_\delta >1$. Taking $q=6p$, using $(a+b)^6\leq32(a^6+b^6)$, and then using Minkowski in $L^p$ gives
\[
 \left\|
  \frac1T\sum_{t=1}^T\norm{g_t}^6\right\|_{L^p}
  \leq\frac1T\sum_{t=1}^T \left\|\norm{g_t}^6\right\|_{L^p}
  =\frac1T\sum_{t=1}^T\left(\E\norm{g_t}^{6p}\right)^{1/p}
  \leq C\left\{(LD)^6+\sigma^6p^3\right\}.
\]
If $Y:=T^{-1}\sum_t\norm{g_t}^6$, Markov's inequality applied to $Y^p$ gives \(\Pp\{Y>e\norm{Y}_{L^p}\}\leq e^{-p}=\delta/4\). Absorbing the numerical factor $e$ into $C_6$ proves the claim.
\end{proof}

For the later burn-in result, we retain a uniform-in-time event.
Define $H_{T,\delta}:=LD+\sigma\sqrt{\log(8T/\delta)}$
Lemma~\ref{lem:subg-consequences}(i), a union bound, and
\eqref{eq:mean-gradient-bound} show that
\begin{equation}
 \max_{t\leq T}\norm{g_t}
 =\max_{t\leq T}\norm{\mathsf G_t}_F
 \leq H_{T,\delta}
 \label{eq:gradient-event}
\end{equation}
with probability at least $1-\delta/4$.

\begin{lemma}[Uniform Shampoo control for later one-step arguments]
\label{lem:Q-control}
For every realization and every $t\leq T$,
\begin{equation}
 0\prec Q_t\preceq I_d.
 \label{eq:Q-pathwise-spectral}
\end{equation}
On the event \eqref{eq:gradient-event},
\begin{equation}
 \norm{I_d-Q_t}_{\op}
 \leq\frac{H_{T,\delta}^2}{2\varepsilon}
 \qquad(1\leq t\leq T).
 \label{eq:Q-uniform-bound}
\end{equation}
\end{lemma}

\begin{proof}
The bias-corrected second-moment factors satisfy $D_t^L=\sum_{s=1}^tq_{t,s}^{(2)} \mathsf G_s\mathsf G_s^{\mathsf T}$ and $D_t^R=\sum_{s=1}^tq_{t,s}^{(2)}\mathsf G_s^{\mathsf T}\mathsf G_s$ where the coefficients $q_{t,s}^{(2)}:=\frac{(1-\beta_2)\beta_2^{t-s}}{1-\beta_2^t}$ are nonnegative and sum to one. Hence, $D_t^L,D_t^R\succeq0$. Since $\varepsilon>0$, every eigenvalue of $Q_t=\left(I_n+\dfrac{D_t^R}{\varepsilon}\right)^{-1/4} \otimes \left(I_m+\dfrac{D_t^L}{\varepsilon}\right)^{-1/4}$ lies in $(0,1]$, proving $0\prec Q_t\preceq I_d$.

On \eqref{eq:gradient-event}, for every matrix $G$, $GG^{\mathsf T}\preceq\norm G_F^2I$ and $G^{\mathsf T}G\preceq\norm G_F^2I$, so $0\preceq D_t^L\preceq H_{T,\delta}^2I_m$ and $0\preceq D_t^R\preceq H_{T,\delta}^2I_n$. Thus every eigenvalue of $Q_t$ is at least $(1+H_{T,\delta}^2/\varepsilon)^{-1/2}$. Therefore
\[
 \norm{I_d-Q_t}_{\op}
 \leq1-\left(1+\frac{H_{T,\delta}^2}{\varepsilon}\right)^{-1/2}
 \leq\frac{H_{T,\delta}^2}{2\varepsilon}.
\]
\end{proof}

\begin{lemma}[Ridge control of the Shampoo geometry]
    \label{lem:ridge-preconditioner}
    Suppose that, for every \(t\),
\(
\lambda_{\min}(D_t^L)\ge \underline\lambda_L,
\,
\lambda_{\min}(D_t^R)\ge \underline\lambda_R,
\) for some $\underline\lambda_L, \underline\lambda_R \geq 0$ where $D_t^L, D_t^R$ are explained in Algorithm~\ref{alg:shampoo}.
Then, for every \(z\in\mathbb R^{mn}\),
\[
\|Q_t(\varepsilon)z\|^2
\le
\frac{\varepsilon}
{\sqrt{
(\varepsilon+\underline\lambda_R)
(\varepsilon+\underline\lambda_L)
}}
\|z\|^2 = \kappa_\varepsilon \|z\|^2.
\]
\end{lemma}
\begin{proof}
For each $t$, let \(D_t^L:=U_t\Lambda_t^L U_t^\top\) and \(D_t^R := V_t\Lambda_t^R V_t^\top\) where \(\Lambda_t^L=\operatorname{diag}(\lambda_{1,t}^L,\ldots,\lambda_{m,t}^L)\) and \(\Lambda_t^R=\operatorname{diag} (\lambda_{1,t}^R,\ldots,\lambda_{n,t}^R)\). The vectors \(\left\{v_{j,t}\otimes u_{i,t}: 1\le i\le m,\; 1\le j\le n \right\}\) form an orthonormal basis of $\mathbb R^{mn}$. In this basis, the eigenvalues of $Q_t(\varepsilon)$ are
\[
    q_{ij,t}(\varepsilon)= \left(1+\frac{\lambda_{j,t}^R}{\varepsilon}\right)^{-1/4}
        \left(1+\frac{\lambda_{i,t}^L}{\varepsilon}\right)^{-1/4} = 
        \left(\frac{\varepsilon}{\sqrt{(\varepsilon+\lambda_{j,t}^R)(\varepsilon+\lambda_{i,t}^L)}}\right)^{1/2}.
\]
Parseval’s identity gives, for every $z\in\mathbb R^{mn}$,
\[
\|Q_t(\varepsilon)z\|^2 =
    \sum_{i=1}^m\sum_{j=1}^n\frac{\left|\left\langle z,v_{j,t}\otimes u_{i,t}\right\rangle\right|^2}{\sqrt{\left(1+\lambda_{j,t}^R/\varepsilon\right)\left(1+\lambda_{i,t}^L/\varepsilon\right)}}
    \leq \frac{\varepsilon}{\sqrt{(\varepsilon+\underline\lambda_{R})(\varepsilon+\underline\lambda_{L})}}\|z\|^2 = \kappa_\varepsilon \|z\|^2.
\]
\end{proof}

\begin{lemma}[Uniform $\ell_2$ Hardy bound for bias-corrected EMA]
\label{lem:hardy}
For every finite Hilbert-space-valued sequence $x_1,\ldots,x_T$,
\begin{equation}
 \sum_{t=1}^T
 \left\|\sum_{s=1}^tw_{t,s}x_s\right\|^2
 \leq6\sum_{s=1}^T\norm{x_s}^2.
 \label{eq:hardy}
\end{equation}
\end{lemma}

\begin{proof}
For $\beta=0$, the averaging operator is the identity. Suppose $0<\beta<1$, and put $q:=\beta^{-1}$, $\lambda_s:=q^{s-1}$, and $\Lambda_t:=\sum_{j=1}^t\lambda_j$. Then, $w_{t,s}=\frac{\lambda_s}{\Lambda_t}\mathbf1\{s\leq t\}$. Choose the Schur weights $p_s:=\Lambda_s^{-1/2}$. Since $\lambda_s=\Lambda_s-\Lambda_{s-1}$ and $\Lambda_0=0$, the row sums telescope:
\begin{equation}
 \sum_{s=1}^tw_{t,s}p_s =\frac1{\Lambda_t}\sum_{s=1}^t\frac{\lambda_s}{\sqrt{\Lambda_s}}   \leq \frac1{\Lambda_t}\sum_{s=1}^t 2(\sqrt{\Lambda_s}-\sqrt{\Lambda_{s-1}}) \leq2p_t.
 \label{eq:schur-row}
\end{equation}
For a fixed column $s$,
\[
 \sum_{t=s}^{\infty}w_{t,s}p_t
 =\lambda_s\sum_{t=s}^{\infty}\Lambda_t^{-3/2} 
 \leq p_s + \sum_{t=s+1}^\infty \frac{\Lambda_t-\Lambda_{t-1}}{\Lambda_t^{3/2}} 
 \leq p_s + \sum_{t=s+1}^\infty 2\left(
  \frac1{\sqrt{\Lambda_{t-1}}}-\frac1{\sqrt{\Lambda_t}}
 \right)   \leq 3p_s
\]
To complete the argument, we use Schur's test. Suppose that
$K_{t,s}\geq0$ and that positive weights $p_s$ satisfy
\(
\sum_s K_{t,s}p_s\leq Rp_t,
\,
\sum_t K_{t,s}p_t\leq Cp_s.
\)
For each fixed $t$, insert the weights $p_s$ and apply
Cauchy--Schwarz:
\[
\begin{aligned}
\left|\sum_s K_{t,s}x_s\right|^2
&=
\left|
\sum_s
\sqrt{K_{t,s}p_s}\,
\sqrt{\frac{K_{t,s}}{p_s}}\,x_s
\right|^2 \leq
\left(\sum_s K_{t,s}p_s\right)
\left(\sum_s K_{t,s}\frac{|x_s|^2}{p_s}\right) \leq
Rp_t\sum_s K_{t,s}\frac{|x_s|^2}{p_s}.
\end{aligned}
\]
Summing over $t$, interchanging the finite sums, and then using
the column bound gives
\[
\begin{aligned}
\sum_t\left|\sum_s K_{t,s}x_s\right|^2
&\leq
R\sum_t p_t\sum_s K_{t,s}\frac{|x_s|^2}{p_s}=
R\sum_s\frac{|x_s|^2}{p_s}
\sum_t K_{t,s}p_t\leq
R\sum_s\frac{|x_s|^2}{p_s}(Cp_s)
=RC\sum_s|x_s|^2.
\end{aligned}
\]
Taking square roots shows that
$\|Kx\|_{\ell_2}\leq\sqrt{RC}\,\|x\|_{\ell_2}$,
which is precisely the asserted operator-norm bound.

Apply this result with $K_{t,s}=w_{t,s}$ for $s\leq t$
and $K_{t,s}=0$ otherwise. The established row and column
bounds give $R=2$ and $C=3$, respectively, so $RC=6$. Finally, for Hilbert-space-valued $x_s$, nonnegativity of the
weights and the triangle inequality give
\(
\left\|\sum_s K_{t,s}x_s\right\|
\leq\sum_s K_{t,s}\norm{x_s}.
\)
Applying the scalar bound above to the sequence
$(\norm{x_s})_{s=1}^T$ therefore yields
\(
\sum_{t=1}^T
\left\|\sum_{s=1}^t w_{t,s}x_s\right\|^2
\leq6\sum_{s=1}^T\norm{x_s}^2,
\)
as required.
\end{proof}

\begin{corollary}[Uniform $\ell_3$ bound for bias-corrected EMA]
\label{cor:hardy-l3}
For $\rho\in[0,1)$, define the scalar bias-corrected EMA operator
\begin{equation}
 (K_\rho x)_t
 :=\sum_{s=1}^t
 \frac{(1-\rho)\rho^{t-s}}{1-\rho^t}x_s,
 \qquad 1\leq t\leq T.
 \label{eq:K-rho}
\end{equation}
Then
\(
 \sum_{t=1}^T |(K_\rho x)_t|^3
 \leq6\sum_{t=1}^T|x_t|^3.
 \label{eq:hardy-l3}
\)
The constant is uniform over $T$ and $\rho\in[0,1)$.
\end{corollary}

\begin{proof}
For a finite sequence $x=(x_1,\ldots,x_T)$, write
\(
\|x\|_{\ell_p}
:=\left(\sum_{t=1}^T|x_t|^p\right)^{1/p}
\,
\|x\|_{\ell_\infty}:=\max_{t\leq T}|x_t|.
\)
The induced operator norm is
$\|K\|_{\ell_p\to\ell_p}
:=\sup_{x\ne0}\|Kx\|_{\ell_p}/\|x\|_{\ell_p}$. Applying Lemma~\ref{lem:hardy} with EMA parameter $\beta=\rho$
gives
\(
\sum_{t=1}^T|(K_\rho x)_t|^2
\leq6\sum_{t=1}^T|x_t|^2.
\)
Taking square roots shows that
$\|K_\rho x\|_{\ell_2}\leq\sqrt6\,\|x\|_{\ell_2}$ for every $x$,
and hence $\|K_\rho\|_{\ell_2\to\ell_2}\leq\sqrt6$. Next, the weights defining $K_\rho$ in \eqref{eq:K-rho}
are nonnegative and sum to one in each row. Therefore,
\[
|(K_\rho x)_t|
\leq
\sum_{s=1}^t
\frac{(1-\rho)\rho^{t-s}}{1-\rho^t}|x_s|
\leq\|x\|_{\ell_\infty}.
\]
Taking the maximum over $t$ gives
$\|K_\rho\|_{\ell_\infty\to\ell_\infty}\leq1$. Riesz--Thorin interpolation combines these two operator bounds.
The interpolation parameter $\vartheta$ corresponding to
$\ell_3$ satisfies $\vartheta=1/3$. Consequently,
\(
\|K_\rho\|_{\ell_3\to\ell_3}
\leq
\|K_\rho\|_{\ell_2\to\ell_2}^{\,2/3}
\|K_\rho\|_{\ell_\infty\to\ell_\infty}^{\,1/3}
\leq(\sqrt6)^{2/3}=6^{1/3}.
\)
Thus, for every $x$,
\(
\sum_{t=1}^T|(K_\rho x)_t|^3
=\|K_\rho x\|_{\ell_3}^3
\leq6\|x\|_{\ell_3}^3
=6\sum_{t=1}^T|x_t|^3.
\)
\end{proof}

\begin{lemma}[Accumulated Shampoo perturbation]
\label{lem:accumulated-preconditioner}
Define $\mathsf P_T:=\sum_{t=1}^T\norm{(I_d-Q_t)\nu_t}^2$, where $\nu_t$ is as defined in \eqref{eq:ab-recursion}. Then, $\mathsf P_T\leq\dfrac{3}{2\varepsilon^2}\sum_{t=1}^T\norm{g_t}^6$. Consequently, on the event in Lemma~\ref{lem:sixth-moment}, $\dfrac{\mathsf P_T}{T}\leq\dfrac{3\mathfrak G_\delta}{2\varepsilon^2}$.
\end{lemma}

\begin{proof}
Set $X_s:=\norm{g_s}^2$ and define the two scalar EMA sequences $U_t:=\sum_{s=1}^tq_{t,s}^{(2)}X_s$ and $V_t:=\sum_{s=1}^tw_{t,s}X_s$. Because $D_t^L=\sum_{s=1}^tq_{t,s}^{(2)} \mathsf G_s\mathsf G_s^{\mathsf T}$, $D_t^R=\sum_{s=1}^tq_{t,s}^{(2)} \mathsf G_s^{\mathsf T}\mathsf G_s$, and $GG^{\mathsf T}\preceq\norm G_F^2I$, $G^{\mathsf T}G\preceq\norm G_F^2I$, we have pathwise $0\preceq D_t^L\preceq U_tI_m$, $0\preceq D_t^R\preceq U_tI_n$. Each normalized factor in $Q_t$ has eigenvalues in
$[(1+U_t/\varepsilon)^{-1/4},1]$. Since the eigenvalues of a
Kronecker product are pairwise products of the factor eigenvalues,
\(
(1+U_t/\varepsilon)^{-1/2}
\leq\lambda_{\min}(Q_t)\leq\lambda_{\max}(Q_t)\leq1.
\)
Thus, $I_d-Q_t\succeq0$, and its eigenvalues are
$1-\lambda_i(Q_t)$. Its operator norm therefore equals its largest
eigenvalue, so that
\[
\norm{I_d-Q_t}_{\op}
=\lambda_{\max}(I_d-Q_t)
=1-\lambda_{\min}(Q_t)
\leq1-(1+U_t/\varepsilon)^{-1/2}.
\]
Using \(1-(1+x)^{-1/2}\leq x/2\) for \(x\geq0\) gives \(\norm{I_d-Q_t}_{\op}\leq U_t/(2\varepsilon)\). Next, Jensen's inequality gives the squared norm momentum bound. We should note that, \(
     \norm{\nu_t}^2
     =\left\|\sum_{s=1}^tw_{t,s}g_s\right\|^2
     \leq\sum_{s=1}^tw_{t,s}\norm{g_s}^2
     =V_t
\)
Combining these two, and Corollary~\ref{cor:hardy-l3},
\[
 \mathsf P_T \leq \frac1{4\varepsilon^2}\sum_{t=1}^TU_t^2V_t
 \leq\frac1{4\varepsilon^2} \left\{\frac23\sum_{t=1}^TU_t^3 + \frac13\sum_{t=1}^TV_t^3\right\}
 \leq \frac6{4\varepsilon^2} \sum_{t=1}^TX_t^3  = \frac3{2\varepsilon^2} \sum_{t=1}^T \norm{g_t}^6.
\]
\end{proof}

\begin{lemma}[Momentum energy with a $\beta$-free signal constant]
\label{lem:energy-split}
Define the cumulative squared momentum, tracking error, and
filtered-noise energies, respectively, by
\[
\mathsf N_T:=\sum_{t=1}^T\norm{\nu_t}^2,
\qquad
\mathsf E_T:=\sum_{t=1}^T\norm{e_t}^2,
\qquad
\mathsf Z_T:=\sum_{t=1}^T\norm{z_t}^2,
\]
where $\nu_t$, $e_t$, and $z_t$ are defined in
\eqref{eq:ab-recursion}, \eqref{eq:energies}, and
\eqref{eq:signal-noise-split}, respectively.
If each $h_t$ is $L$-Lipschitz and satisfies
$h_t(\theta_t^\star)=0$, then, pathwise,
\(
\mathsf N_T\leq12L^2\mathsf E_T+2\mathsf Z_T.
\)
\end{lemma}

\begin{proof}
Since $\nu_t=u_t+z_t$, the inequality
$\norm{x+y}^2\leq2\norm{x}^2+2\norm{y}^2$ gives
\(
\mathsf N_T
=\sum_{t=1}^T\norm{u_t+z_t}^2
\leq2\sum_{t=1}^T\norm{u_t}^2+2\mathsf Z_T.
\)
For each realization, apply Lemma~\ref{lem:hardy} to the vector
sequence $x_s=h_s(\theta_s)$. This yields
\[
\sum_{t=1}^T\norm{u_t}^2
=
\sum_{t=1}^T
\left\|\sum_{s=1}^t w_{t,s}h_s(\theta_s)\right\|^2
\leq6\sum_{s=1}^T\norm{h_s(\theta_s)}^2.
\]
The root property and Lipschitz continuity imply, for every $s$,
\(
\norm{h_s(\theta_s)}
=
\norm{h_s(\theta_s)-h_s(\theta_s^\star)}
\leq L\norm{\theta_s-\theta_s^\star}
=L\norm{e_s}.
\)
Squaring and summing therefore gives
\(
\sum_{t=1}^T\norm{u_t}^2
\leq6L^2\mathsf E_T.
\)
Substituting into the first inequality proves
\[
\mathsf N_T
\leq2(6L^2\mathsf E_T)+2\mathsf Z_T
=12L^2\mathsf E_T+2\mathsf Z_T.
\]
\end{proof}

\begin{lemma}[High-probability filtered-noise energy]
\label{lem:filtered-energy}
Let $\Psi_{T,\beta}:=\dfrac1{T(1-\beta^2)}\sum_{t=1}^Tb_t^2$, where $b_t$ is defined in \eqref{eq:ab-recursion}. For every $\delta\in(0,1)$, with probability at least $1-\delta$,
\begin{equation}
 \frac{\mathsf Z_T}{T}
 \leq{} \mathfrak Z_{T,\beta,\delta} := 2\sigma^2(\log2)\Psi_{T,\beta}
 +\frac{2\sigma^2}{T(1-\beta^2)}\log\frac4\delta
 +\frac{16C_0\sigma^2\beta^2}{T(1-\beta^2)^2} \log\frac4\delta.
 \label{eq:Z-highprob}
\end{equation}
\end{lemma}

\begin{proof}
The filtered noise satisfies $z_t=a_tz_{t-1}+b_t\xi_t$ where $z_0 = 0$. Squaring and summing gives
\begin{equation}
 \mathsf Z_T
 =\sum_{t=1}^Ta_t^2\norm{z_{t-1}}^2
  +\sum_{t=1}^Tb_t^2\norm{\xi_t}^2
  +2\sum_{t=1}^Ta_tb_t\ip{z_{t-1}}{\xi_t}.
 \label{eq:z-square-sum}
\end{equation}
Because $0\leq a_t\leq\beta$ and $0<b_t^2\leq1$, conditional Jensen applied to the concave map $x\mapsto x^{b_t^2}$ yields 
\(\E\!\left[\exp\!\left(\dfrac{b_t^2\norm{\xi_t}^2}{\sigma^2}\right)\middle|\F_{t-1}\right]
 \leq \left(\E[e^{\norm{\xi_t}^2/\sigma^2}\mid\F_{t-1}]\right)^{b_t^2}
 \leq2^{b_t^2}\)
Iterating conditional expectations and applying Markov's inequality shows that, with probability at least $1-\delta/2$,
\begin{equation}
 \sum_{t=1}^Tb_t^2\norm{\xi_t}^2\leq\sigma^2\left\{
  (\log2)\sum_{t=1}^Tb_t^2+\log\frac2\delta
 \right\}.
 \label{eq:V-bound}
\end{equation}

Suppose first that $\beta>0$. The vector
$2a_tb_tz_{t-1}$ is $\F_{t-1}$-measurable. Moreover, $\sum_{t=1}^T\norm{2a_tb_tz_{t-1}}^2 \leq4\beta^2\mathsf Z_T$. Apply Lemma~\ref{lem:subg-consequences}(ii) with tuning parameter $(1-\beta^2)/(8\beta^2)$ and failure probability $\delta/2$. Then
\begin{equation}
 \Bigg|2\sum_{t=1}^Ta_tb_t\ip{z_{t-1}}{\xi_t}\Bigg|
 \leq\frac{1-\beta^2}{2}\mathsf Z_T
 +\frac{8C_0\sigma^2\beta^2}{1-\beta^2}\log\frac4\delta.
 \label{eq:M-bound}
\end{equation}
Combining \eqref{eq:V-bound}, and \eqref{eq:M-bound}, and noting $\sum_{t=1}^Ta_t^2\norm{z_{t-1}}^2\leq\beta^2\mathsf Z_T$ and enlarging $\log(2/\delta)$ to $\log(4/\delta)$, gives
\[
 \mathsf Z_T
 \leq\frac{2\sigma^2}{1-\beta^2}\left\{
  (\log2)\sum_{t=1}^Tb_t^2+\log\frac4\delta
 \right\}
 +\frac{16C_0\sigma^2\beta^2}{(1-\beta^2)^2}\log\frac4\delta.
\]
Division by $T$ proves the result.  If $\beta=0$, then $z_t=\xi_t$ and
the cross term vanishes; the stated bound follows directly from \eqref{eq:V-bound} after a
factor-two relaxation.
\end{proof}

\begin{lemma}[High-probability summation-by-parts boundary]
\label{lem:boundary}
Define $\mathsf B_T:=\sum_{t=1}^{T-1}\beta^t\norm{\nu_t} + d_T\norm{\nu_T}$ and $\mathfrak B_\delta:=C_B\left\{LD+\sigma\sqrt{\ell_\delta}\right\}$. There is a universal numerical constant $C_B>0$ such that, with probability at least $1-\delta/4$, $\mathsf B_T \leq c_\beta\mathfrak B_\delta$. For $\beta=0$, both sides are zero.
\end{lemma}

\begin{proof}
If $\beta=0$, then $d_t=0$ and $\mathsf B_T=0$ identically,
so the claim is immediate. Assume henceforth that $\beta>0$,
and let $p:=\ell_\delta=\log(4/\delta)>1$. Taking expectations in \eqref{eq:q-moment} gives
\(
\left(\E\norm{\xi_s}^p\right)^{1/p}\leq C\sigma\sqrt p.
\)
Since $g_s=h_s(\theta_s)+\xi_s$ and
$\norm{h_s(\theta_s)}\leq LD$ by
\eqref{eq:mean-gradient-bound}, Minkowski's inequality yields
\[
\|\norm{g_s}\|_{L^p}
\leq
\|\norm{h_s(\theta_s)}\|_{L^p}
+\|\norm{\xi_s}\|_{L^p}
\leq LD+C\sigma\sqrt p.
\]
This bound is uniform in $s$. Next, by \eqref{eq:ab-recursion},
$\nu_t=\sum_{s=1}^t w_{t,s}g_s$, where the weights are
deterministic, nonnegative, and sum to one. The triangle
inequality followed by Minkowski's inequality therefore gives
\[
\begin{aligned}
\|\norm{\nu_t}\|_{L^p}
&\leq
\left\|\sum_{s=1}^t w_{t,s}\norm{g_s}\right\|_{L^p}\leq
\sum_{s=1}^t w_{t,s}\|\norm{g_s}\|_{L^p}
\leq LD+C\sigma\sqrt p.
\end{aligned}
\]
Applying Minkowski once more to the definition of $\mathsf B_T$,
we obtain,
\[
\begin{aligned}
\|\mathsf B_T\|_{L^p}
&\leq
\sum_{t=1}^{T-1}\beta^t\|\norm{\nu_t}\|_{L^p}
+d_T\|\norm{\nu_T}\|_{L^p}\leq
\left(\sum_{t=1}^{T-1}\beta^t+d_T\right)
\left(LD+C\sigma\sqrt p\right)\leq
2c_\beta\left(LD+C\sigma\sqrt p\right),
\end{aligned}
\]
where the last step uses
$\sum_{t=1}^{T-1}\beta^t\leq c_\beta$ and $d_T\leq c_\beta$.

If $\|\mathsf B_T\|_{L^p}=0$, the conclusion is immediate.
Otherwise, Markov's inequality applied to the nonnegative
random variable $\mathsf B_T^p$ gives
\(
\Pp\!\left\{\mathsf B_T>e\|\mathsf B_T\|_{L^p}\right\}
\leq
\frac{\E[\mathsf B_T^p]}
     {e^p\|\mathsf B_T\|_{L^p}^p}
=e^{-p}
=\frac{\delta}{4}.
\)
Consequently, with probability at least $1-\delta/4$,
\[
\mathsf B_T
\leq2e\,c_\beta\left(LD+C\sigma\sqrt{\ell_\delta}\right)
\leq c_\beta\mathfrak B_\delta,
\]
where the final inequality follows by choosing the universal
constant $C_B$ sufficiently large.
\end{proof}

\begin{lemma}[Refined accumulated momentum pairing under total squared path variation]
\label{lem:pairing}
Suppose \eqref{eq:squared-path-variation} holds. Then, pathwise,
\begin{align}
 \sum_{t=1}^T\ip{e_t}{Q_t\nu_t}
 \geq{}\frac{7\mu}{8}\mathsf E_T
 +\sum_{t=1}^T\ip{e_t}{\xi_t}
 -\alpha \left(c_\beta\sqrt{\kappa_\varepsilon} + \frac{A_\beta}{2}\right)\mathsf N_T
 -\frac{c_\beta^2}{2\alpha A_\beta}\mathcal D_T - D\mathsf B_T
 -\frac{2}{\mu}\mathsf P_T.
 \label{eq:pairing-final}
\end{align}
\end{lemma}

\begin{proof}
Since $a_t+b_t=1$, recursion \eqref{eq:ab-recursion} implies \(\nu_t-g_t=d_t(\nu_{t-1}-\nu_t)\). Using $d_1=0$ and $d_{t+1}-d_t=\beta^t$,
\begin{align}
  \sum_{t=1}^T\ip{e_t}{\nu_t-g_t} ={}& \sum_{t=1}^Td_t\ip{e_t}{\nu_{t-1}-\nu_t} = {}\sum_{t=1}^{T-1} \ip{d_{t+1}e_{t+1}-d_te_t}{\nu_t}-d_T\ip{e_T}{\nu_T}  \notag\\
 ={}&\sum_{t=1}^{T-1}d_{t+1}\ip{e_{t+1}-e_t}{\nu_t}+\sum_{t=1}^{T-1}(d_{t+1}-d_t)\ip{e_t}{\nu_t} -d_T\ip{e_T}{\nu_T} \notag\\
 \implies  \sum_{t=1}^T\ip{e_t}{\nu_t} ={}& \sum_{t=1}^T\ip{e_t}{g_t} + \sum_{t=1}^{T-1}d_{t+1}\ip{e_{t+1}-e_t}{\nu_t}+\sum_{t=1}^{T-1}\beta^t\ip{e_t}{\nu_t}-d_T\ip{e_T}{\nu_T}. \label{eq:sbp}
\end{align}
Since $\theta_t\in\Theta$ and projection is nonexpansive, 
\(\norm{\theta_{t+1}-\theta_t} \leq\alpha\norm{Q_t\nu_t}\leq\alpha\sqrt{\kappa_\varepsilon}\norm{\nu_t}\), where the last inequality uses Lemma~\ref{lem:ridge-preconditioner}.  Also, \(e_{t+1}-e_t = (\theta_{t+1}-\theta_t) + (\theta_t^\star-\theta_{t+1}^\star)\), so \(\norm{e_{t+1}-e_t} \leq\alpha\sqrt{\kappa_\varepsilon}\norm{\nu_t}+\Delta_t\). We write $\Delta_t:=\norm{\theta_{t+1}^\star-\theta_t^\star}$. Using $d_{t+1}\leq c_\beta$ gives
\begin{align*}
 \sum_{t=1}^{T-1}d_{t+1}\ip{e_{t+1}-e_t}{\nu_t}
 \geq-c_\beta\alpha \sqrt{\kappa_\varepsilon} \mathsf N_T - c_\beta\sum_{t=1}^{T-1}\Delta_t\norm{\nu_t}
 \geq-c_\beta\alpha\sqrt{\kappa_\varepsilon}\mathsf N_T - c_\beta\sqrt{\mathcal D_T\mathsf N_T}
\end{align*}
where the last line uses Cauchy--Schwarz and $\sum_t\Delta_t^2\leq\mathcal D_T$. Since $\norm{e_t}\leq D$,
\begin{align*}
 \left|\sum_{t=1}^{T-1}\beta^t\ip{e_t}{\nu_t}\right| +|d_T\ip{e_T}{\nu_T}|
 \leq D\left\{\sum_{t=1}^{T-1}\beta^t\norm{\nu_t} +d_T\norm{\nu_T}\right\} = D\mathsf B_T.
\end{align*}
Strong monotonicity and the root property imply \(\ip{e_t}{g_t} = \ip{e_t}{h_t(\theta_t)}+\ip{e_t}{\xi_t} \geq \mu\norm{e_t}^2+\ip{e_t}{\xi_t}\). Substituting into~\eqref{eq:sbp} yields
\begin{align}
 \sum_{t=1}^T\ip{e_t}{\nu_t} \geq{} \mu\mathsf E_T+\sum_{t=1}^T\ip{e_t}{\xi_t} -\alpha c_\beta\sqrt{\kappa_\varepsilon}\mathsf N_T -c_\beta\sqrt{\mathcal D_T\mathsf N_T} -D\mathsf B_T.
 \label{eq:unpreconditioned-pairing}
\end{align}
For the preconditioner perturbation, Young's inequality gives, pathwise,
\[
    \left|\ip{e_t}{(Q_t-I_d)\nu_t}\right| \leq\norm{e_t}\norm{(I_d-Q_t)\nu_t} \leq\frac\mu8\norm{e_t}^2+\frac2\mu\norm{(I_d-Q_t)\nu_t}^2.
\]
Summing this inequality and using $c_\beta\sqrt{\mathcal D_T\mathsf N_T} \leq \frac{\alpha A_\beta}{2}\mathsf N_T + \frac{c_\beta^2}{2\alpha A_\beta}\mathcal D_T$ proves \eqref{eq:pairing-final}.
\end{proof}

\begin{lemma}[Summed projected recursion]
\label{lem:projected}
If $\mu\alpha\leq1$ and $\gamma:=\mu\alpha/8$, then pathwise
\begin{equation}
 2\alpha(1+\gamma)\sum_{t=1}^T\ip{e_t}{Q_t\nu_t}\leq{} 2\norm{e_1}^2+\gamma\mathsf E_T + (1+\gamma)\alpha^2\kappa_\varepsilon\mathsf N_T + \frac{9\mathcal D_T}{\mu\alpha}.
 \label{eq:projected-sum}
\end{equation}
\end{lemma}

\begin{proof}
Let $s_t:=\theta_t^\star-\theta_{t+1}^\star$, define $\Delta_t := \norm{s_t}$. Since $\theta_t^\star\in\Theta$ and the projection is nonexpansive, \(\norm{\theta_{t+1}-\theta_t^\star}^2\leq\norm{e_t-\alpha Q_t\nu_t}^2\). Also $e_{t+1}=(\theta_{t+1}-\theta_t^\star)+s_t$. The inequality
$\norm{x+y}^2\leq(1+\gamma)\norm x^2+(1+\gamma^{-1})\norm y^2$ gives
\[  \norm{e_{t+1}}^2    \leq{}
    (1+\gamma)\norm{e_t}^2-2\alpha(1+\gamma)\ip{e_t}{Q_t\nu_t}+(1+\gamma)\alpha^2\norm{Q_t\nu_t}^2 +(1+\gamma^{-1})\Delta_t^2.
\]
By Lemma~\ref{lem:ridge-preconditioner}, $\norm{Q_t\nu_t}^2\leq\kappa_\varepsilon\norm{\nu_t}^2$. Summing and using \((1+\gamma)\sum_{t=1}^T\norm{e_t}^2- \sum_{t=1}^T\norm{e_{t+1}}^2 \leq2\norm{e_1}^2+\gamma\mathsf E_T\) together with \(1+\gamma^{-1}=1+\frac8{\mu\alpha}\leq\frac9{\mu\alpha}\) proves the claim.
\end{proof}

\begin{proof}[Proof of Theorem~\ref{thm:main-tracking}]
We work on the intersection of four events controlling the
random terms that appear in the tracking argument:
\begin{enumerate}[label=(\roman*)]
\item The energy of the noise retained by the momentum filter
is controlled:
\(
 \frac{\mathsf Z_T}{T}
 \leq\mathfrak Z_{T,\beta,\delta/4}
 \lesssim\mathfrak Z_{T,\beta,\delta}.
\)
This follows from Lemma~\ref{lem:filtered-energy}, applied
with failure probability $\delta/4$.

\item The accumulated pairing between the tracking error
and the current noise is controlled:
\(
 \left|\sum_{t=1}^T\ip{e_t}{\xi_t}\right|
 \leq\frac{\mu}{16}\mathsf E_T
      +\frac{16C_0\sigma^2}{\mu}\log\frac8\delta.
\)
Since $e_t$ is $\F_{t-1}$-measurable, this follows from
Lemma~\ref{lem:subg-consequences}(ii), with $v_t=e_t$,
$a=\mu/16$, and failure probability $\delta/4$.

\item The sixth-moment average of the observed gradients
is controlled:
\(
 \frac1T\sum_{t=1}^T\norm{g_t}^6\leq\mathfrak G_\delta.
\)
Lemma~\ref{lem:sixth-moment} guarantees this event with
failure probability at most $\delta/4$. Through
Lemma~\ref{lem:accumulated-preconditioner}, it also controls
the Shampoo perturbation.

\item The momentum boundary term arising from summation
by parts is controlled:
\(
 \mathsf B_T\leq c_\beta\mathfrak B_\delta.
\)
Lemma~\ref{lem:boundary} guarantees this event with failure
probability at most $\delta/4$.
\end{enumerate}
By the union bound, their intersection has probability
at least $1-\delta$. All subsequent inequalities are
considered on this intersection. Moreover, bound of $\alpha$ and $\mu\leq L$ imply $\mu\alpha<\mu^2/(384L^2)\leq1/384<1$, so the hypothesis $\mu\alpha\leq1$ of Lemma~\ref{lem:projected} is satisfied. Insert \eqref{eq:pairing-final} into \eqref{eq:projected-sum}. Since $\gamma=\mu\alpha/8$, $2\alpha(1+\gamma)\dfrac{7\mu}{8}\mathsf E_T - \gamma\mathsf E_T = \mu\alpha\left\{\dfrac74(1+\gamma)-\dfrac18\right\}\mathsf E_T \geq\dfrac32\mu\alpha\mathsf E_T$. Using $1+\gamma\leq2$ and $A_\beta\geq1$ for the remaining terms gives
\[
 \frac32\mu\alpha\mathsf E_T \leq{} 2\norm{e_1}^2
 +\left\{\frac9{\mu\alpha}
 +\frac{2c_\beta^2}{A_\beta}\right\}\mathcal D_T
 +\alpha^2K_{\beta,\varepsilon}\mathsf N_T
 +\frac{8\alpha}{\mu}\mathsf P_T
 +4\alpha D\mathsf B_T
 +4\alpha\left|\sum_{t=1}^T\ip{e_t}{\xi_t}\right|
\]
where $K_{\beta,\varepsilon} := 4c_\beta\sqrt{\kappa_\varepsilon}+ 2A_\beta + 2\kappa_\varepsilon$. Lemma~\ref{lem:energy-split} gives $\mathsf N_T\leq12L^2\mathsf E_T+2\mathsf Z_T$, so $\alpha^2K_{\beta,\varepsilon}\mathsf N_T \leq 12\alpha^2K_{\beta,\varepsilon} L^2\mathsf E_T+2\alpha^2K_{\beta,\varepsilon}\mathsf Z_T$. Because $A_\beta=(1-\beta)^{-1}$ and $\kappa_\varepsilon\le 1$, $12\alpha^2K_{\beta,\varepsilon} L^2 \leq 96\alpha\dfrac{\alpha L^2}{1-\beta}<\dfrac14\mu\alpha$. On event (ii), Lemma~\ref{lem:subg-consequences}(ii) gives $\left|\sum_{t=1}^T\ip{e_t}{\xi_t}\right| \leq\dfrac\mu{16}\mathsf E_T + \dfrac{16C_0\sigma^2}{\mu}\log\dfrac8\delta$. Using these, we obtain
\[
 \mu\alpha\mathsf E_T \leq{} 2\norm{e_1}^2
 +\left\{\frac9{\mu\alpha}
 +\frac{2c_\beta^2}{A_\beta}\right\}\mathcal D_T
 +2\alpha^2K_{\beta,\varepsilon}\mathsf Z_T
 +\frac{8\alpha}{\mu}\mathsf P_T
 +4\alpha D\mathsf B_T
 +\frac{64C_0\alpha\sigma^2}{\mu}\log\frac8\delta.
\]
Divide by $\mu\alpha T$. On event (i),  $\mathsf Z_T/T\leq\mathfrak Z_{T,\beta,\delta/4}
\leq2\mathfrak Z_{T,\beta,\delta}$. On event (iii), Lemma~\ref{lem:accumulated-preconditioner} gives $\frac{\mathsf P_T}{T}\leq\frac{3\mathfrak G_\delta}{2\varepsilon^2}$ and on event (iv), $\mathsf B_T\leq c_\beta\mathfrak B_\delta$. Hence
\[
 \frac{\mathsf E_T}{T}
 \lesssim \Bigg\{\frac{\norm{e_1}^2}{\mu\alpha T}
 +\frac{\mathcal D_T}{\mu^2\alpha^2T}
 +\frac{c_\beta^2}{A_\beta}\frac{\mathcal D_T}{\mu\alpha T}
 +\frac{\alpha K_{\beta,\varepsilon}}{\mu}  \mathfrak Z_{T,\beta,\delta}
 +\frac{\sigma^2}{\mu^2T}\log\frac8\delta
 +\frac{c_\beta D\mathfrak B_\delta}{\mu T}
 +\frac{\mathfrak G_\delta}{\mu^2\varepsilon^2}\Bigg\}.
\]
Substituting $c_\beta$ and $A_\beta$ prove the average tracking error bound.
\end{proof}

\subsection{Proof of Theorem~\ref{thm:minimax}}

\begin{corollary}[Stepsize optimization with an arbitrary finite ridge]
\label{cor:optimized-upper}
Under the assumptions of Theorem~\ref{thm:main-tracking}, fix a
deterministic $\varepsilon>0$ and let $K_{\beta,\varepsilon}$ be
the deterministic coefficient in the refined tracking bound.
Let
\[
 A_1:=\frac{\norm{W_1-W_1^\star}_F^2}{\mu T}
      +\frac{\beta^2}{1-\beta}\frac{\mathcal D_T}{\mu T},
 \qquad
 A_2:=\frac{\mathcal D_T}{\mu^2T},
 \qquad
 A_3(\varepsilon):=
 \frac{K_{\beta,\varepsilon}\mathfrak Z_{T,\beta,\delta}}{\mu}.
\]
Assume $A_1+A_2>0$ and let $\bar\alpha_\star(\varepsilon)$ be
the unique positive solution of
\[
 A_3(\varepsilon)\bar\alpha_\star(\varepsilon)^3
 -A_1\bar\alpha_\star(\varepsilon)-2A_2=0.
\]
Suppose
$\bar\alpha_\star(\varepsilon)<\alpha_{\max}(\beta)
:=\mu(1-\beta)/(384L^2)$.
Choose $\alpha=\bar\alpha_\star(\varepsilon)$. Then, with
probability at least $1-\delta$,
\begin{equation}
 \begin{aligned}
 \frac1T\sum_{t=1}^T\norm{W_t-W_t^\star}_F^2
 \leq C\Bigg\{&
 \frac{\norm{W_1-W_1^\star}_F
       \sqrt{K_{\beta,\varepsilon}
             \mathfrak Z_{T,\beta,\delta}}}
      {\mu\sqrt T}
 +\frac{\beta}{\sqrt{1-\beta}}
   \frac{\sqrt{\mathcal D_TK_{\beta,\varepsilon}
               \mathfrak Z_{T,\beta,\delta}}}
        {\mu\sqrt T}\\
 &+\frac{\mathcal D_T^{1/3}
          (K_{\beta,\varepsilon}
           \mathfrak Z_{T,\beta,\delta})^{2/3}}
        {\mu^{4/3}T^{1/3}}
 +\frac{\sigma^2}{\mu^2T}\log\frac8\delta
 +\frac{\beta}{1-\beta}
       \frac{D\mathfrak B_\delta}{\mu T}
 +\frac{\mathfrak G_\delta}{\mu^2\varepsilon^2}
 \Bigg\}.
 \end{aligned}
 \label{eq:optimized-upper}
\end{equation}
Moreover,
\begin{equation}
 \bar\alpha_\star(\varepsilon)\asymp
 \max\left\{
 \left(
  \frac{\norm{W_1-W_1^\star}_F^2
        +\dfrac{\beta^2}{1-\beta}\mathcal D_T}
       {T K_{\beta,\varepsilon}
          \mathfrak Z_{T,\beta,\delta}}
 \right)^{1/2},
 \left(
  \frac{\mathcal D_T}
       {\mu T K_{\beta,\varepsilon}
          \mathfrak Z_{T,\beta,\delta}}
 \right)^{1/3}
 \right\},
 \label{eq:alpha-rate}
\end{equation}
where the comparison constants are universal.
\end{corollary}

\begin{proof}
For the fixed ridge $\varepsilon$, the $\alpha$-dependent part
of \eqref{eq:main-tracking-bound} is
$\Phi_\varepsilon(\alpha)
:=A_1/\alpha+A_2/\alpha^2+A_3(\varepsilon)\alpha$,
$\alpha>0$. Differentiating,
\[
 \Phi_\varepsilon'(\alpha)
 =-\frac{A_1}{\alpha^2}-\frac{2A_2}{\alpha^3}
   +A_3(\varepsilon),
 \qquad
 \Phi_\varepsilon''(\alpha)
 =\frac{2A_1}{\alpha^3}+\frac{6A_2}{\alpha^4}>0.
\]
Because $A_1+A_2>0$,
$\Phi_\varepsilon'(\alpha)\to-\infty$ as $\alpha\downarrow0$,
while $\Phi_\varepsilon'(\alpha)\to A_3(\varepsilon)>0$
as $\alpha\to\infty$. Hence $\Phi_\varepsilon$ has a unique
minimizer. For the rate of the minimizer, set
$x:=\sqrt{A_1/A_3(\varepsilon)}$,
$y:=(2A_2/A_3(\varepsilon))^{1/3}$, and
$m:=\max\{x,y\}$. The cubic becomes
$\bar\alpha_\star(\varepsilon)^3
=x^2\bar\alpha_\star(\varepsilon)+y^3$.
If $\bar\alpha_\star(\varepsilon)<x$, then
$\bar\alpha_\star(\varepsilon)^3
<x^2\bar\alpha_\star(\varepsilon)$, a contradiction;
similarly $\bar\alpha_\star(\varepsilon)\geq y$.
Thus $\bar\alpha_\star(\varepsilon)\geq m$.
At $\alpha=2m$,
$x^2(2m)+y^3\leq2m^3+m^3<(2m)^3$.
Uniqueness of the positive root therefore gives
$m\leq\bar\alpha_\star(\varepsilon)\leq2m$.
Substituting the definitions of $A_1,A_2,A_3(\varepsilon)$
proves \eqref{eq:alpha-rate}.
Since $\Phi_\varepsilon(\alpha)$ dominates both
$A_1/\alpha+A_3(\varepsilon)\alpha$ and
$A_2/\alpha^2+A_3(\varepsilon)\alpha$, we conclude
\[
 \begin{aligned}
 \inf_{\alpha>0}\Phi_\varepsilon(\alpha)
 &\geq
 \max\left\{
 2\sqrt{A_1A_3(\varepsilon)},
 \frac{3}{2^{2/3}}
 A_2^{1/3}A_3(\varepsilon)^{2/3}
 \right\} \geq
 c\left(
 \sqrt{A_1A_3(\varepsilon)}
 +A_2^{1/3}A_3(\varepsilon)^{2/3}
 \right)
 \end{aligned}
\]
for a universal $c>0$. Conversely, at $\alpha=m$,
$A_1/m=A_3(\varepsilon)x^2/m\leq A_3(\varepsilon)m$ and
$A_2/m^2=A_3(\varepsilon)y^3/(2m^2)
\leq A_3(\varepsilon)m/2$. Since
\[
 A_3(\varepsilon)m
 =\max\left\{
 \sqrt{A_1A_3(\varepsilon)},
 2^{1/3}A_2^{1/3}A_3(\varepsilon)^{2/3}
 \right\},
\]
this gives
$\Phi_\varepsilon(m)\leq
C\{\sqrt{A_1A_3(\varepsilon)}
+A_2^{1/3}A_3(\varepsilon)^{2/3}\}$.
Therefore,
\begin{equation}
 \inf_{\alpha>0}\Phi_\varepsilon(\alpha)
 \asymp
 \sqrt{A_1A_3(\varepsilon)}
 +A_2^{1/3}A_3(\varepsilon)^{2/3}.
 \label{eq:Phi-optimum}
\end{equation}
Using $\sqrt{u+v}\leq\sqrt u+\sqrt v$,
$\sqrt{A_1A_3(\varepsilon)}$ yields the first two terms
in \eqref{eq:optimized-upper}, while
$A_2^{1/3}A_3(\varepsilon)^{2/3}$ yields the third.
Adding the $\alpha$-independent terms from
Theorem~\ref{thm:main-tracking} proves the result.
\end{proof}
\begin{remark}[Feasibility of the optimized stepsize]
As in Theorem~\ref{thm:minimax}, suppose
$0\leq\beta\leq\bar\beta<1$. Since
$0<\kappa_\varepsilon\leq1$,
\(
 \frac{2}{1-\beta}
 \leq K_{\beta,\varepsilon}
 \leq\frac{4+2\beta}{1-\beta}
 \leq\frac{6}{1-\bar\beta}.
\)
If $T^{-1/2}\lesssim\mathcal D_T=o(T)$, then
\eqref{eq:alpha-rate} and
$\mathfrak Z_{T,\beta,\delta}
\asymp \frac{1-\beta}{1+\beta}\sigma^2$ imply, uniformly in
$\varepsilon>0$,
\[
 \bar\alpha_\star(\varepsilon)
 =O\!\left((\mathcal D_T/T)^{1/3}\right)=o(1),
 \qquad
 \alpha_{\max}(\beta)
 \geq\frac{\mu(1-\bar\beta)}{384L^2}>0,
\]
for fixed problem parameters and confidence level.
Hence $\bar\alpha_\star(\varepsilon)<\alpha_{\max}(\beta)$
holds automatically for all sufficiently large $T$, including
for a deterministic ridge sequence $\varepsilon=\varepsilon_T$.
When $\mathcal D_T=\Theta(T)$, feasibility depends on the
proportionality constant and the problem parameters; the exact
condition
$\bar\alpha_\star(\varepsilon)<\alpha_{\max}(\beta)
=\mu(1-\beta)/(384L^2)$ is imposed.
\end{remark}

\begin{lemma}[Discrete triangular functions]
\label{lem:triangles}
Let $B=\{\ell,\ell+1,\ldots,r\}$ be a consecutive block of length $n=r-\ell+1\geq8$. There is a function $\phi_B:B\to[0,1]$ such that
\[
 \phi_B(\ell)=\phi_B(r)=0, \qquad
 |\phi_B(t+1)-\phi_B(t)|\leq\frac3n, \qquad
 \sum_{t=\ell}^{r-1}\{\phi_B(t+1)-\phi_B(t)\}^2\leq\frac9n, \qquad \sum_{t\in B}\phi_B(t)^2\geq\frac n{32}.
\]
\end{lemma}

\begin{proof}
Let $q=\lfloor(n-1)/2\rfloor$ and define
$\phi_B(t):=\min\{t-\ell,r-t\}/q$.
The numerator is zero at both endpoints and lies between $0$
and $q$, so $\phi_B(\ell)=\phi_B(r)=0$ and $\phi_B(t)\in[0,1]$.

The numerator changes by at most one between consecutive points.
Since $q\geq(n-2)/2$ and $n\geq8$, we get
\(
|\phi_B(t+1)-\phi_B(t)|
\leq\frac1q\leq\frac2{n-2}\leq\frac3n.
\)
There are $n-1$ such increments, and hence
\[
\sum_{t=\ell}^{r-1}
\{\phi_B(t+1)-\phi_B(t)\}^2
\leq(n-1)\frac9{n^2}\leq\frac9n.
\]

On the two sides of the triangle, each level $k/q$,
for $k=0,\ldots,q-1$, occurs at the distinct points
$t=\ell+k$ and $t=r-k$. Therefore,
\[
\sum_{t\in B}\phi_B(t)^2
\geq\frac2{q^2}\sum_{k=0}^{q-1}k^2
=\frac{(q-1)(2q-1)}{3q}
\geq\frac q3
\geq\frac{n-2}{6}
\geq\frac n8.
\]
The last three inequalities use $q\geq3$,
$q\geq(n-2)/2$, and $n\geq8$, respectively.
This is stronger than the required lower bound $n/32$.
\end{proof}

For $v,w\in\{-1,+1\}^J$, we define their Hamming distance by
\(
d_H(v,w):=\sum_{j=1}^J\mathbf1\{v_j\neq w_j\}.
\)
\begin{lemma}[Binary packing]
\label{lem:VG}
There are universal constants $c_V,c_H>0$ such that, for every sufficiently large $J$, a subset $\cU\subset\{-1,+1\}^J$ exists with $\log|\cU|\geq c_VJ$ and $d_H(v,w)\geq c_HJ$ for $v\neq w$.
\end{lemma}

\begin{proof}
Let $m:=\lfloor J/4\rfloor$. Starting with the entire Hamming
cube, repeatedly choose a remaining point and delete all remaining
points within Hamming distance $m$ of it. Continue until no points
remain, and denote the selected set by $\cU$. Each newly selected point lies outside the balls deleted earlier.
Thus distinct $v,w\in\cU$ satisfy
\(
d_H(v,w)\geq m+1>J/4\geq J/8.
\)

A Hamming ball of radius $m$ contains
\(
V_J:=\sum_{k=0}^m\binom Jk
\)
points, since exactly $k$ coordinates can be changed in
$\binom Jk$ ways. To bound this sum, take $a:=1/3$.
For $k\leq m\leq J/4$, we have $a^k\geq a^{J/4}$, so the
binomial theorem gives
\[
V_J
\leq a^{-J/4}\sum_{k=0}^J\binom Jk a^k
=a^{-J/4}(1+a)^J
=\exp\{Jh(1/4)\},
\]
where $h(p):=-p\log p-(1-p)\log(1-p)$. Each selection deletes at most $V_J$ points, and all $2^J$
points have been deleted when the construction terminates.
Consequently,
\(
|\cU|V_J\geq2^J
\), that implies, \(
\log|\cU|
\geq J\{\log2-h(1/4)\}.
\)
Since $h(1/4)<\log2$, the claimed bounds hold with smaller universal positive constants $c_H,c_V$.
\end{proof}

\begin{proof}[Proof of Theorem~\ref{thm:minimax}]

\textbf{Proof of Minimax lower bound}
\medskip\noindent\emph{Step 1: admissible paths.}
Choose $W^\circ\in\operatorname{int}(\mathcal W)$ and set
$\widetilde W^0:=(W^0+W^\circ)/2$. By convexity,
\[
 \widetilde W^0+zU
 =\tfrac12W^\circ+\tfrac12(W^0+2zU)
 \in\operatorname{int}(\mathcal W),
 \qquad |z|\leq R/2.
\]
Thus we have an interior segment on which to construct the paths. For sufficiently large $T$, set
\(
 J:=\left\lfloor\min\{J_\star,T/16\}\right\rfloor.
\)
Since $\mathcal D_T\gtrsim T^{-1/2}$ and the problem parameters
are fixed, $J\to\infty$. Hence $J$ is large enough for
Lemma~\ref{lem:VG}, and $8\leq J\leq T/16$.
Partition $\{1,\ldots,T\}$ into $J$ consecutive blocks
$B_1,\ldots,B_J$ whose lengths differ by at most one.
With $h:=\lfloor T/J\rfloor$, each block has length $h$ or
$h+1$, and $h\geq T/(2J)\geq8$.
Let $\phi_j$ be the triangular function on $B_j$ from
Lemma~\ref{lem:triangles}, and let $\cU$ be the packing from
Lemma~\ref{lem:VG}.

For a universal constant $0<c_a\leq1/18$, define
\begin{equation}
 a^2:=\min\left\{
 c_a\frac{\sigma^2J}{\mu^2T},\frac{R^2}{16}
 \right\},
 \qquad
 W_t^{\star,v}:=\widetilde W^0+av_j\phi_j(t)U,
 \quad t\in B_j,\quad v\in\cU.
 \label{eq:packed-path}
\end{equation}
For notational ease, we extend $W_{T+1}^{\star,v}:=\widetilde W^0$. Because $a\leq R/4$, every optimizer lies in the interior
segment above. Each block starts and ends at
$\widetilde W^0$, so there is no movement between blocks
and $W_1^{\star,v}=\widetilde W^0$.
Lemma~\ref{lem:triangles} therefore gives
\[
\begin{aligned}
 \sum_{t=1}^{T}
 \norm{W_{t+1}^{\star,v}-W_t^{\star,v}}_F^2
 &\leq\frac{9a^2J}{h}
 \leq\frac{18a^2J^2}{T}\leq18c_a\frac{\sigma^2J^3}{\mu^2T^2}
 \leq18c_a\mathcal D_T
 \leq\mathcal D_T,
\end{aligned}
\]
where we used $J\leq J_\star$ and
$J_\star^3=\mu^2\mathcal D_TT^2/\sigma^2$. Finally, $\mathcal D_T\lesssim T$ bounds both $J_\star/T$
and $r_T^\star$. Consequently, the caps above give
$J\gtrsim J_\star$ and $a^2\gtrsim r_T^\star$, with implicit
constants depending on the fixed problem parameters and
the constant in $\mathcal D_T\lesssim T$.

\medskip\noindent\emph{Step 2: separation.}
If $v_j\neq w_j$, then $\sum_{t\in B_j}\norm{W_t^{\star,v}-W_t^{\star,w}}_F^2 = 4a^2\sum_{t\in B_j}\phi_j(t)^2 \geq\dfrac{a^2h}{8}$. Using Lemma \ref{lem:VG} and $hJ\geq T/2$, there is a universal $c_2>0$ such that, for distinct $v,w\in\cU$, $\dfrac1T\sum_{t=1}^T \norm{W_t^{\star,v}-W_t^{\star,w}}_F^2\geq c_2a^2$.

\medskip\noindent\emph{Step 3: an admissible sample-loss experiment.}
Fix \(v\in\cU\). Let
\(\zeta_1,\ldots,\zeta_T\stackrel{\mathrm{iid}}{\sim}
N(0,\sigma^2/4)\), independently of the algorithm's private
randomization, and define
\(
 X_t^v:=W_t^{\star,v}-\frac{\zeta_t}{\mu}U,
 \,
 g(W,X):=\frac{\mu}{2}\norm{W-X}_F^2-\frac{\sigma^2}{8\mu}.
\)
The population objective induced by this sample loss is
\[
 F_t^v(W):=\E_v[g(W,X_t^v)]
 =\frac{\mu}{2}\norm{W-W_t^{\star,v}}_F^2 +\frac{\E\zeta_t^2}{2\mu}-\frac{\sigma^2}{8\mu} 
 =\frac{\mu}{2}\norm{W-W_t^{\star,v}}_F^2
\]
where the cross term vanishes because \(\E\zeta_t=0\),
\(\norm{U}_F=1\), and \(\E\zeta_t^2=\sigma^2/4\). The constant in
\(g\) makes this equality exact; it does not affect the sample gradient. At the algorithm's query \(W_t\), the sample gradient is
\(
 \mathsf G_t:=\nabla_W g(W_t,X_t^v)
 =\mu(W_t-W_t^{\star,v})+\zeta_tU
 =\mathsf H_t^v(W_t)+\Xi_t,
\)
where
\(
 \mathsf H_t^v(W):=\mu(W-W_t^{\star,v}),
 \,
 \Xi_t:=\zeta_tU.
\)
The map \(\mathsf H_t^v\) has root \(W_t^{\star,v}\), is
\(\mu\)-strongly monotone and \(\mu\)-Lipschitz, and therefore is
\(L\)-Lipschitz because \(\mu\leq L\). The roots and mean maps are
deterministic under the fixed experiment \(v\), hence predictable.
The fresh Gaussian draws give
\[
 \E[\Xi_t\mid\F_{t-1}]=0,
 \qquad
 \E\!\left[
   \exp\!\left(\frac{\norm{\Xi_t}_F^2}{\sigma^2}\right)
   \middle|\F_{t-1}
 \right]
 =
 \E e^{\zeta_t^2/\sigma^2}
 =\sqrt{2}\leq2.
\]
Step~1 verifies the common initial root, feasibility, and squared
path-variation budget. Thus the complete sample-loss experiment belongs
to \(\M_T\).

Let \(S\) denote the algorithm's private random seed, independent of \(v\) and the samples. Given \(S\) and \(X_{1:T}^v\), one can reproduce any causal algorithm's trajectory recursively: start from its prescribed
\(W_1\), compute
\(
 \mathsf G_t=\nabla_W g(W_t,X_t^v)=\mu(W_t-X_t^v),
\)
and use the algorithm's rule to obtain its next query from \(S\) and \(\mathsf G_{1:t}\). Thus every output of the algorithm is a function of \((S,X_{1:T}^v)\). Since \(S\) is independent of \(v\), a lower bound for estimators given the entire sequence \(X_{1:T}^v\) also applies to every causal algorithm.

\medskip\noindent\emph{Step 4: Kullback--Leibler divergence.}
Let \(P_v\) be the law of \(X_{1:T}^v\). All differences between the mean paths lie in the direction \(U\). In particular, for \(t\in B_j\), the scalar coordinate
\(
 \ipF{X_t^v-\widetilde W^0}{U}
 =av_j\phi_j(t)-\frac{\zeta_t}{\mu}
\)
is Gaussian with variance \(\sigma^2/(4\mu^2)\); the components orthogonal to \(U\) are the same under every \(v\). Independence over time and the Gaussian relative-entropy formula give
\[
 \KL(P_v\Vert P_w) =\frac{1}{2\{\sigma^2/(4\mu^2)\}} \sum_{t=1}^T \norm{W_t^{\star,v}-W_t^{\star,w}}_F^2
 \leq\frac{8\mu^2a^2T}{\sigma^2} \leq 8c_aJ.
\]
Since \(|v_j-w_j|\leq2\), \(0\leq\phi_j(t)\leq1\), and \(\norm{U}_F=1\), each squared distance in the sum is at most \(4a^2\). The packing satisfies \(\log|\cU|\geq c_VJ\). Choosing \(c_a\leq\min\{1/18,c_V/128\}\), consistently with Step~1, yields
\(
 \max_{v,w\in\cU}\KL(P_v\Vert P_w)
 \leq\frac{1}{16}\log|\cU|.
\)

\medskip\noindent\emph{Step 5: Fano reduction.}
Let \(P_S\) be the law of the algorithm's private seed \(S\), and define \(\widetilde P_v:=P_v\otimes P_S\), the joint law of \((X_{1:T},S)\) under hypothesis \(v\). Let \(V\) be uniform on \(\cU\), and conditional on \(V=v\), draw \((X_{1:T},S)\sim\widetilde P_v\). We write $|\mathcal U|$ for the cardinality of a finite set
$\mathcal U$.
For random variables $V,Y$ with joint law $P_{V,Y}$ and
marginal laws $P_V,P_Y$, their mutual information is
\(
I(V;Y):=D_{\mathrm{KL}}\!\left(
P_{V,Y}\,\middle\|\,P_V\otimes P_Y
\right).
\) Because \(S\) is independent of both \(V\) and the sample sequence, the pairwise-KL bound implies
\[
 I(V;S,X_{1:T})= I(V;X_{1:T}) \leq\frac{1}{|\cU|^2} \sum_{v,w\in\cU}\KL(P_v\Vert P_w)
 \leq\frac{1}{16}\log|\cU|.
\]
Every decoder constructed from the algorithm can be written as
\(\widehat V=f(X_{1:T},S)\) and takes values in \(\cU\). Applying Fano's inequality to the observation \((X_{1:T},S)\) gives
\(
 \Pp(\widehat V\neq V)
 \geq
 1-\frac{I(V;X_{1:T})+\log2}{\log|\cU|}.
\)
Since \(J\to\infty\), so, for sufficiently large \(T\), \(\log|\cU|\geq16\log2/3\). It follows that
\(
 \Pp(\widehat V\neq V)
 \geq1-\frac1{16}-\frac3{16}
 =\frac34.
\)

Now take an arbitrary path estimate \(\widehat W_{1:T}\) produced from the observations and the private seed. Construct a decoder by choosing a nearest packed path, breaking ties deterministically:
\(
 \widehat V\in
 \mathop{\arg\min}_{w\in\cU}
 \text{dist}_T(\widehat W,W^{\star,w}),
 \) where \(
 \text{dist}_T(X,Y)^2:=
 \frac1T\sum_{t=1}^T\norm{X_t-Y_t}_F^2.
\)
This decoder takes values in \(\cU\), as required by Fano's inequality. If the true index is \(v\) and
\(
 \text{dist}_T(\widehat W,W^{\star,v})^2<\frac{c_2a^2}{4},
\)
then Step~2 and the triangle inequality make \(v\) the unique nearest index. Hence a decoding error entails tracking error at least \(c_2a^2/4\), and for all large enough $T$,
\[
 \frac1{|\cU|}\sum_{v\in\cU}
 \widetilde P_v\!\left\{
  \text{dist}_T(\widehat W,W^{\star,v})^2
  \geq\frac{c_2a^2}{4}
 \right\}
 \geq\frac34.
\]
In particular, the maximum of these probabilities is at least \(3/4\), and hence at least \(1/2\). By Step~1, \(a^2\gtrsim r_T^\star\). Because the algorithm was arbitrary and every \(P_v\) belongs to \(\M_T\), this proves the claimed minimax lower bound.

\medskip\noindent\textbf{Matching Minimax upper bound.}
Fix $\delta\in(0,1/2)$ and $0\leq\beta\leq\bar\beta<1$.
Throughout, the problem parameters are fixed, and
$\kappa T^{-1/2}\leq\mathcal D_T\leq KT$ for fixed
$\kappa,K>0$. Lemma~\ref{lem:filtered-energy} gives
$\mathfrak Z_{T,\beta,\delta}
\asymp_{\bar\beta,\delta}\sigma^2$, while
$\mathfrak B_\delta$ and $\mathfrak G_\delta$ from
Lemmas~\ref{lem:boundary} and \ref{lem:sixth-moment}
are independent of $T$. Use a deterministic coefficient
$K_{\beta,\varepsilon}$ valid uniformly over $\mathcal M_T$;
for the original class, one may take
$\underline\lambda_L=\underline\lambda_R=0$, so
$\kappa_\varepsilon=1$. In every case,
\(
 \frac{2}{1-\beta}
 \leq K_{\beta,\varepsilon}
 \leq\frac{4+2\beta}{1-\beta}
 \leq\frac{6}{1-\bar\beta},
\)
For uniform tuning over $\mathcal M_T$, replace the initial
squared error in $A_1$ by its diameter bound:
\[
 A_1:=\frac{D^2}{\mu T}
      +\frac{\beta^2}{1-\beta}
       \frac{\mathcal D_T}{\mu T}.
\]
For each deterministic $\varepsilon>0$, keep
$A_2,A_3(\varepsilon)$ as defined in
Corollary~\ref{cor:optimized-upper}, and let
$\bar\alpha_\star(\varepsilon)$ minimize
$\Phi_\varepsilon(\alpha)
:=A_1/\alpha+A_2/\alpha^2+A_3(\varepsilon)\alpha$
over $\alpha>0$. Choose
\(
 \alpha_\star(\varepsilon)
 :=\min\{\bar\alpha_\star(\varepsilon),
          \alpha_{\max}(\beta)/2\}.
\)
The optimizer estimate in \eqref{eq:alpha-rate}, with the
initial error replaced by $D^2$, and the uniform bounds on
$K_{\beta,\varepsilon}$ give
$\bar\alpha_\star(\varepsilon)\leq M<\infty$, uniformly in
$T,\varepsilon,P$ and $\beta\leq\bar\beta$.
Here $M$ may depend on the fixed parameters,
$K,\delta,\bar\beta$, but not on $T,\varepsilon$ or $P$.
Since
$\alpha_{\max}(\beta)
\geq\mu(1-\bar\beta)/(384L^2)$, define
\(
 \theta:=
 \min\left\{1,\frac{\mu(1-\bar\beta)}{768L^2M}\right\}>0.
\)
Then
$\theta\bar\alpha_\star(\varepsilon)
\leq\alpha_\star(\varepsilon)
\leq\bar\alpha_\star(\varepsilon)$.
Thus the cap preserves the optimized bound up to a fixed,
possibly parameter-dependent factor:
\[
 \begin{aligned}
 \Phi_\varepsilon(\alpha_\star(\varepsilon))
 &\leq
 \theta^{-2}
 \Phi_\varepsilon(\bar\alpha_\star(\varepsilon))\lesssim
 \theta^{-2}\left(
 \sqrt{A_1A_3(\varepsilon)}
 +A_2^{1/3}A_3(\varepsilon)^{2/3}
 \right)\lesssim
 \theta^{-2}\left(
 \frac{D\sigma}{\mu\sqrt T}
 +\frac{\sigma}{\mu}\sqrt{\frac{\mathcal D_T}{T}}
 +r_T^\star
 \right).
 \end{aligned}
\]
The last comparison is uniform in $\varepsilon$, because
$K_{\beta,\varepsilon}$ is bounded above and below as displayed
above and $\beta\leq\bar\beta<1$. The two comparisons needed
at the endpoints are
\[
 \frac{D\sigma/(\mu\sqrt T)}{r_T^\star}
 \leq
 \frac{D\mu^{1/3}}{\sigma^{1/3}\kappa^{1/3}},
 \qquad
 \frac{(\sigma/\mu)\sqrt{\mathcal D_T/T}}{r_T^\star}
 =\sqrt{\frac{J_\star}{T}}
 \leq\left(\frac{\mu^2K}{\sigma^2}\right)^{1/6}.
\]
Consequently,
$\Phi_\varepsilon(\alpha_\star(\varepsilon))
\leq C_{\mathrm{opt}}r_T^\star$, where
$C_{\mathrm{opt}}$ is independent of $T,\varepsilon,P$ but
may depend on the fixed problem parameters,
$\kappa,K,\delta,\bar\beta$.
The stochastic-fluctuation and momentum-boundary terms in
\eqref{eq:main-tracking-bound} satisfy
\[
 \frac{\sigma^2}{\mu^2T}\log\frac8\delta
 +\frac{\beta}{1-\beta}
       \frac{D\mathfrak B_\delta}{\mu T}
 =O(T^{-1})=O(r_T^\star),
\]
because $\delta$ is fixed, $\beta\leq\bar\beta$, and
$r_T^\star\gtrsim T^{-1/2}$. Finally, choose any deterministic
finite ridge satisfying
\[
 \varepsilon_\star\geq
 \frac{\sqrt{\mathfrak G_\delta}}
      {\mu\sqrt{r_T^\star}}
 =
 \frac{\sqrt{\mathfrak G_\delta}}
      {\mu^{1/3}\sigma^{2/3}}
 \left(\frac{T}{\mathcal D_T}\right)^{1/6},
 \qquad
 \alpha_\star:=\alpha_\star(\varepsilon_\star).
\]
Then
$\mathfrak G_\delta/(\mu^2\varepsilon_\star^2)
\leq r_T^\star$. The preceding stepsize estimates hold uniformly
in $\varepsilon$, so they apply to this ridge choice.
Applying Theorem~\ref{thm:main-tracking} with these parameters
gives, for every $P\in\mathcal M_T$ and sufficiently large $T$,
\[
 \mathbb P_P\left\{
  \frac1T\sum_{t=1}^T
  \|W_t^{\mathrm{Shampoo}}-W_t^\star\|_F^2
  >Cr_T^\star
 \right\}\leq\delta.
\]
All tuning choices and the constant $C$ are common to the class.
The constant may depend on the fixed parameters,
$\kappa,K,\delta,\bar\beta$, but is independent of $T$ and $P$.
Taking the supremum over $P$ proves the minimax upper bound.
If additionally $T^{-1/2}=o(\mathcal D_T)$ and
$\mathcal D_T=o(T)$, then
$\bar\alpha_\star(\varepsilon_\star)\to0$, so the cap is
eventually inactive. Both endpoint comparison ratios displayed
above tend to zero. Thus $C$ can be chosen independent of the
problem parameters for sufficiently large $T$, while retaining
possible dependence on the fixed $\delta,\bar\beta$.
\end{proof}

\subsection{Proof of Theorem~\ref{thm:burnin}}

\begin{proof}
Write $\beta:=\beta_1$. Use $e_t$ from \eqref{eq:energies},
the weights $w_{t,s}$ and momentum $\nu_t$ from
\eqref{eq:ab-recursion}, and $u_t,z_t$ from
\eqref{eq:signal-noise-split}. Throughout, $C$ denotes a
universal numerical constant that may change between inequalities.
Since $\mu\leq L$ and $\alpha\leq\mu/(8L^2)$, we have 
\(
 0<\alpha\mu\leq\frac18.\)

\medskip\noindent
\emph{Step 1: basic moment bounds.}
Conditional centering and \eqref{eq:mean-gradient-bound} give
\(
 \E[\norm{g_t}^2\mid\F_{t-1}]
 =\norm{h_t(\theta_t)}^2
  +\E[\norm{\xi_t}^2\mid\F_{t-1}]
 \leq (LD)^2+\sigma^2\log2.
\)
Also, \eqref{eq:q-moment} gives
$\E\norm{\xi_t}^6\leq12\sigma^6$. Therefore,
\(
 \E\norm{g_t}^6
 \leq32\left\{(LD)^6+12\sigma^6\right\}
 \leq C\left\{(LD)^6+\sigma^6\right\}.
\)
Because $\nu_t$ is a convex combination of the gradients,
Jensen's inequality yields
\[
 \E\norm{\nu_t}^2
 \leq\sum_{s=1}^tw_{t,s}\E\norm{g_s}^2
 \leq (LD)^2+\sigma^2.
\]
\noindent
\emph{Step 2: control the momentum-induced displacement.}
Define
\(
 \mathcal L_t:=\sum_{s=1}^t
 w_{t,s}\norm{e_t-e_s}^2.
\)
Projection nonexpansiveness and \(\|Q_jv\|^2\leq\kappa_\varepsilon\|v\|^2,
\)
from Lemma~\ref{lem:ridge-preconditioner} imply
\(
 \norm{\theta_{j+1}-\theta_j}
 \leq\alpha\norm{Q_j\nu_j}
 \leq\alpha\sqrt{\kappa_\varepsilon}\norm{\nu_j}.
\)
For $s<t$, writing $k:=t-s$ and using the optimizer
movement bound, we obtain
\(
 \norm{e_t-e_s}
 \leq\alpha\sqrt{\kappa_\varepsilon}\sum_{j=s}^{t-1}\norm{\nu_j}+k\Delta.
\)
Squaring, applying Cauchy--Schwarz, and using Step~1 gives
\[
\begin{aligned}
 \E\norm{e_t-e_s}^2
 &\leq2\alpha^2\kappa_\varepsilon k\sum_{j=s}^{t-1}\E\norm{\nu_j}^2
       +2k^2\Delta^2\leq2k^2\left\{
 \alpha^2\kappa_\varepsilon\bigl((LD)^2+\sigma^2\bigr)+\Delta^2
 \right\}.
\end{aligned}
\]
The lag weights form a geometric distribution conditioned
on $k<t$. Since $k^2$ is increasing, its conditional
second moment is at most the untruncated second moment:
\(
 \sum_{s=1}^tw_{t,s}(t-s)^2
 \leq(1-\beta)\sum_{k=0}^\infty k^2\beta^k
 =\frac{\beta(1+\beta)}{(1-\beta)^2}.
\)
For $\beta=0$, this inequality holds with both sides zero.
Consequently,
\(
 \E\mathcal L_t
 \leq
 \frac{2\beta(1+\beta)}{(1-\beta)^2}
 \left\{
 \alpha^2\kappa_\varepsilon\bigl((LD)^2+\sigma^2\bigr)+\Delta^2
 \right\}.
\)

\medskip\noindent
\emph{Step 3: control the finite-ridge perturbation.}
Define \(\widetilde p_t:=(I_d-Q_t)u_t\), so that
\(Q_t\nu_t=u_t+Q_tz_t-\widetilde p_t\), and let
\(U_t:=\sum_{s=1}^tq_{t,s}^{(2)}\norm{g_s}^2\) and
\(V_t^h:=\sum_{s=1}^tw_{t,s}\norm{h_s(\theta_s)}^2\).
The spectral estimate in Lemma~\ref{lem:accumulated-preconditioner}
and Jensen's inequality give
\(\norm{\widetilde p_t}^2\leq U_t^2V_t^h/(4\varepsilon^2)\).
Using \(a^2b\leq(2a^3+b^3)/3\) and Jensen yields
\[
 \norm{\widetilde p_t}^2
 \leq\frac1{12\varepsilon^2}
 \left\{
 2\sum_{s=1}^tq_{t,s}^{(2)}\norm{g_s}^6
 +\sum_{s=1}^tw_{t,s}\norm{h_s(\theta_s)}^6
 \right\}.
\]
Step~1 and \eqref{eq:mean-gradient-bound} therefore imply
\(
 \E\norm{\widetilde p_t}^2
 \leq\frac{C}{\varepsilon^2}\left\{(LD)^6+\sigma^6\right\}.
\)

\medskip\noindent
\emph{Step 4: establish a pathwise one-step contraction.}
For $s\leq t$, strong monotonicity, the root property,
Lipschitzness, and Young's inequality give
\[
\begin{aligned}
 \ip{e_t}{h_s(\theta_s)}
 &\geq\mu\norm{e_s}^2
       -L\norm{e_t-e_s}\norm{e_s}\geq\frac{\mu}{2}\norm{e_s}^2
       -\frac{L^2}{2\mu}\norm{e_t-e_s}^2\geq\frac{\mu}{4}\norm{e_t}^2
       -\frac{\mu^2+L^2}{2\mu}\norm{e_t-e_s}^2,
\end{aligned}
\]
where the last step uses
$\norm{e_s}^2\geq\frac12\norm{e_t}^2-\norm{e_t-e_s}^2$.
Averaging over $s$ gives
\(
 \ip{e_t}{u_t}
 \geq\frac{\mu}{4}\norm{e_t}^2
      -\frac{\mu^2+L^2}{2\mu}\mathcal L_t.
\)
Similarly, Jensen's inequality and Lipschitzness imply
\(
 \norm{u_t}^2
 \leq L^2\sum_{s=1}^tw_{t,s}\norm{e_s}^2
 \leq2L^2\norm{e_t}^2+2L^2\mathcal L_t.
\)
Expanding the square and using
$2\alpha^2L^2\leq\alpha\mu/4$, we obtain
\[
 \norm{e_t-\alpha u_t}^2
 \leq\left(1-\frac{\alpha\mu}{4}\right)\norm{e_t}^2
      +2\alpha\left(\mu+\frac{L^2}{\mu}\right)\mathcal L_t.
\]

Since \(Q_t\nu_t=u_t+Q_tz_t-\widetilde p_t\), projection
nonexpansiveness, Lemma~\ref{lem:ridge-preconditioner}, and the optimizer
movement bound yield
\(
 \norm{e_{t+1}}
 \leq\norm{e_t-\alpha u_t}
      +\alpha\sqrt{\kappa_\varepsilon}\norm{z_t}
      +\alpha\norm{\widetilde p_t}+\Delta.
\)
Apply Young's inequality with parameter $\alpha\mu/16$,
and use
\[
 \bigl(\alpha\sqrt{\kappa_\varepsilon}\norm{z_t}+\alpha\norm{\widetilde p_t}+\Delta\bigr)^2
 \leq3\left\{
 \alpha^2\kappa_\varepsilon\norm{z_t}^2+\alpha^2\norm{\widetilde p_t}^2+\Delta^2
 \right\}.
\]
The coefficient of $\norm{e_t}^2$ satisfies
\(
 \left(1+\frac{\alpha\mu}{16}\right)
 \left(1-\frac{\alpha\mu}{4}\right)
 =\rho-\frac{(\alpha\mu)^2}{64}\leq\rho.
\)
Since $\alpha\mu\leq1/8$, the remaining coefficients
are bounded by universal numerical multiples, giving
\begin{equation}
\begin{aligned}
 \norm{e_{t+1}}^2
 \leq{}&\rho\norm{e_t}^2
 +C\alpha\left(\mu+\frac{L^2}{\mu}\right)\mathcal L_t+\frac{C\alpha}{\mu}
   \left(\kappa_\varepsilon\norm{z_t}^2+\norm{\widetilde p_t}^2\right)
 +\frac{C\Delta^2}{\alpha\mu}.
\end{aligned}
 \label{eq:error-t}
\end{equation}

\medskip\noindent
\emph{Step 5: unroll the recursion and obtain a probability bound.}
Fix $T\geq1$ and write $N:=T-1$. Define
\(
 Y_T:=
 \sum_{t=1}^{N}\rho^{N-t}
 \left\{
 \alpha\left(\mu+\frac{L^2}{\mu}\right)\mathcal L_t
 +\frac{\alpha}{\mu}
   \left(\kappa_\varepsilon\norm{z_t}^2+\norm{\widetilde p_t}^2\right)
 \right\}.
\)
Iterating \eqref{eq:error-t} and using
$1-\rho=3\alpha\mu/16$ gives
\(
 \norm{e_T}^2
 \leq\rho^N\norm{e_1}^2
      +CY_T+\frac{C\Delta^2}{\alpha^2\mu^2}.
\)

Distinct martingale-difference increments are orthogonal
in expectation. Therefore,
\[
\begin{aligned}
 \E\norm{z_t}^2
 &=\sum_{s=1}^tw_{t,s}^2\E\norm{\xi_s}^2\leq\sigma^2(\log2)
 \frac{1-\beta}{1+\beta}\frac{1+\beta^t}{1-\beta^t}\leq\sigma^2
 \left\{\frac{1-\beta}{1+\beta}
         +\frac{2\beta^t}{1+\beta}\right\}.
\end{aligned}
\]
Here we used $1-\beta^t\geq1-\beta$ in the last step.
Now, we define
\(
 \Psi_T:=
 \frac1{1+\beta}\sum_{t=1}^{N}\rho^{N-t}\beta^t.
\)
The noise contribution to $\E Y_T$ consequently satisfies
\(
 \frac{\alpha}{\mu}\kappa_\varepsilon
 \sum_{t=1}^{N}\rho^{N-t}\E\norm{z_t}^2
 \leq
 \frac{16\kappa_\varepsilon\sigma^2}{3\mu^2}\frac{1-\beta}{1+\beta}
 +\frac{2\alpha\kappa_\varepsilon\sigma^2}{\mu}\Psi_T.
\)

Comparing the resulting terms with the definition of
$\overline{\mathcal R}_\delta$ gives
\(
 \E Y_T
 \leq C\delta\overline{\mathcal R}_\delta
      +\frac{C\alpha\kappa_\varepsilon\sigma^2}{\mu}\Psi_T.
\)
Since $Y_T$ is nonnegative, Markov's inequality gives
$Y_T\leq\E Y_T/\delta$ with probability at least $1-\delta$.
On this event,
\begin{equation}
 \norm{e_T}^2
 \leq\rho^N\norm{e_1}^2
      +C\overline{\mathcal R}_\delta
      +\frac{C\alpha\kappa_\varepsilon\sigma^2}{\mu\delta}\Psi_T.
 \label{eq:tracking-error}
\end{equation}
\noindent
\emph{Step 6: control the two transients.}
For the initialization term,
\(
 \rho^N\norm{e_1}^2
 \leq e^{-(1-\rho)N}\norm{e_1}^2
 \leq\overline{\mathcal R}_\delta
\)
whenever
\(
 N\geq\frac1{1-\rho}
 \log_+\!\left(
 \frac{\norm{e_1}^2}{\overline{\mathcal R}_\delta}
 \right).
\)

For the EMA transient, recall $q=\max\{\rho,\beta\}$.
Each term satisfies
$\rho^{N-t}\beta^t\leq q^{N-t}q^t=q^N$.
Since there are $N$ terms,
\(
 \Psi_T\leq\frac{Nq^N}{1+\beta}\leq Nq^N.
\)
This bound includes the case $\rho=\beta$.
Moreover,
\(
 Nq^N
 \leq Ne^{-(1-q)N}
 \leq\frac{2}{1-q}e^{-(1-q)N/2},
\)
where the last inequality follows from
$x e^{-x/2}\leq2$ for $x\geq0$.
Consequently,
\(
 \frac{\alpha\kappa_\varepsilon\sigma^2}{\mu\delta}\Psi_T
 \leq\frac{2\alpha\kappa_\varepsilon\sigma^2}{\mu\delta(1-q)}
       e^{-(1-q)N/2}
 \leq2\overline{\mathcal R}_\delta
\)
whenever
\[
 N\geq\frac2{1-q}
 \log_+\!\left(
 \frac{\alpha\kappa_\varepsilon \sigma^2}
      {\mu\delta(1-q)\overline{\mathcal R}_\delta}
 \right).
\]
In particular, when $\rho=\beta$, we have $q=\beta$ and the same bound holds for $ \Psi_T$, eventually yielding the same lower bound for $N$. Choose $T_{\mathrm{burn}}$ as
the maximum of these two thresholds. This has the order
stated in the theorem. For each $T\geq T_{\mathrm{burn}}$,
both transient terms in \eqref{eq:tracking-error} are
bounded by universal multiples of
$\overline{\mathcal R}_\delta$.
Using the vectorization isometry, we conclude that
\[
 \Pp\!\left(
 \norm{W_T-W_T^\star}_F^2
 \leq C\overline{\mathcal R}_\delta
 \right)\geq1-\delta.
\]
\end{proof}

\subsection{Proof of Theorem~\ref{thm:nonconvex-shampoo-hp}}
\begin{proof}
\emph{Step 1: express the update as an inexact gradient step.}
Write $\beta:=\beta_1$. Using the mean-gradient identity
$h_t(\theta)=\nabla F_t(\theta)$, abbreviate
$h_t:=h_t(\theta_t)$ and define
$\delta_t:=Q_t\nu_t-h_t$. Thus,
\[
 \theta_{t+1}
 =\Pi_\Theta\bigl(\theta_t-\alpha(h_t+\delta_t)\bigr).
\]
The vector $\delta_t$ collects the discrepancy between
the Shampoo direction and the current population gradient.

\medskip\noindent
\emph{Step 2: obtain a one-step descent inequality.}
Put $s_t:=\theta_{t+1}-\theta_t$.
The projection inequality, applied with
$p=\theta_{t+1}$, $z=\theta_t-\alpha(h_t+\delta_t)$,
and $y=\theta_t$, gives
\(
 0\leq
 \ip{s_t+\alpha(h_t+\delta_t)}{-s_t}
 =-\norm{s_t}^2-\alpha\ip{h_t+\delta_t}{s_t}.
\)
Rearranging and applying Young's inequality yields
\begin{equation}
 \ip{h_t}{s_t}
 \leq-\frac1\alpha\norm{s_t}^2-\ip{\delta_t}{s_t}
 \leq-\frac1{2\alpha}\norm{s_t}^2
       +\frac\alpha2\norm{\delta_t}^2.
 \label{eq:stationarity-projection-descent}
\end{equation}
Since $\nabla F_t$ is $L$-Lipschitz on the convex set
$\Theta$, the descent lemma gives
\[
\begin{aligned}
 F_t(\theta_{t+1})
 &\leq F_t(\theta_t)+\ip{h_t}{s_t}
       +\frac L2\norm{s_t}^2\leq F_t(\theta_t)
       -\left(\frac1{2\alpha}-\frac L2\right)\norm{s_t}^2
       +\frac\alpha2\norm{\delta_t}^2.
\end{aligned}
\]
The stepsize condition $2\alpha L\leq1$ implies the coefficient of $\norm{s_t}^2$ is at least
$1/(4\alpha)$. Consequently,
\(
 \norm{s_t}^2
 \leq4\alpha\bigl(F_t(\theta_t)-F_t(\theta_{t+1})\bigr)
       +2\alpha^2\norm{\delta_t}^2.
\)

\medskip\noindent
\emph{Step 3: compare with the noiseless projected step.}
Define
$\bar\theta_{t+1}:=\Pi_\Theta(\theta_t-\alpha h_t)$.
The inputs to the two projections differ by
$\alpha\delta_t$, so nonexpansiveness gives
\(
 \norm{\theta_{t+1}-\bar\theta_{t+1}}
 \leq\alpha\norm{\delta_t}.
\)
By \eqref{eq:def_grad},
\(
 \alpha\mathcal G_{\alpha,t}(\theta_t)
 =-s_t+(\theta_{t+1}-\bar\theta_{t+1}).
\)
Using $\norm{x+y}^2\leq2\norm{x}^2+2\norm{y}^2$
and the displacement bound from Step~2, we obtain
\begin{equation}
\begin{aligned}
 \norm{\mathcal G_{\alpha,t}(\theta_t)}^2
 &\leq\frac2{\alpha^2}\norm{s_t}^2
       +2\norm{\delta_t}^2\leq\frac8\alpha
       \bigl(F_t(\theta_t)-F_t(\theta_{t+1})\bigr)
       +6\norm{\delta_t}^2.
\end{aligned}
 \label{eq:stationarity-one-step-mapping}
\end{equation}

\medskip\noindent
\emph{Step 4: telescope the objective values.}
Summing the objective differences leaves the initial
and final values, together with the changes in the
objective at intermediate iterates:
\begin{equation}
\begin{aligned}
 &\sum_{t=1}^T
   \bigl(F_t(\theta_t)-F_t(\theta_{t+1})\bigr) =
 F_1(\theta_1)-F_T(\theta_{T+1})
 +\sum_{t=1}^{T-1}
   \bigl(F_{t+1}(\theta_{t+1})-F_t(\theta_{t+1})\bigr)\,\leq F_1(\theta_1)-F_\star+\mathcal V_T.
\end{aligned}
 \label{eq:stationarity-telescope}
\end{equation}
The last inequality uses the uniform lower bound
$F_T(\theta_{T+1})\geq F_\star$ and the positive-part
definition of $\mathcal V_T$ in Section~\ref{sec:projected-stationarity}.
Summing \eqref{eq:stationarity-one-step-mapping}
and dividing by $T$ therefore gives
\begin{equation}
 \frac1T\sum_{t=1}^T
 \norm{\mathcal G_{\alpha,t}(\theta_t)}^2
 \leq
 \frac{8(F_1(\theta_1)-F_\star+\mathcal V_T)}{\alpha T}
 +\frac6T\sum_{t=1}^T\norm{\delta_t}^2.
 \label{eq:stationarity-before-error}
\end{equation}

\medskip\noindent
\emph{Step 5: control the accumulated direction error.}
Use $u_t,z_t$ from \eqref{eq:signal-noise-split}.
Because the weights sum to one, the current-gradient
term cancels, giving
\(
 u_t-h_t
 =\sum_{s=1}^{t-1}
   w_{t,s}\bigl(h_s(\theta_s)-h_t\bigr).
\)
By \eqref{eq:mean-gradient-bound}, both gradients
inside each difference have norm at most $LD$.
Moreover, \eqref{eq:ab-recursion} gives
$\sum_{s<t}w_{t,s}=a_t\leq\beta$. Hence
\(
 \norm{u_t-h_t}\leq2LD\,a_t\leq2LD\beta.
\)
Since $\nu_t-h_t=(u_t-h_t)+z_t$, it follows that
\[
 \frac1T\sum_{t=1}^T\norm{\nu_t-h_t}^2
 \leq8L^2D^2\beta^2+\frac{2\mathsf Z_T}{T}.
\]
The decomposition \(\delta_t=Q_t(\nu_t-h_t)+(Q_t-I_d)h_t\) gives
\(
 \frac1T\sum_{t=1}^T\norm{\delta_t}^2
 \leq16\kappa_\varepsilon L^2D^2\beta^2
      +\frac{4\kappa_\varepsilon\mathsf Z_T}{T}
      +\frac{2\mathsf P_T^h}{T}
\)
where \(\mathsf P_T^h:=\sum_{t=1}^T\norm{(I_d-Q_t)h_t}^2\).
To control this new perturbation energy, write
\(U_t:=\sum_{s=1}^tq_{t,s}^{(2)}\norm{g_s}^2\).
The spectral estimate in Lemma~\ref{lem:accumulated-preconditioner}
and \eqref{eq:mean-gradient-bound} imply
\(
 \norm{(I_d-Q_t)h_t}^2
 \leq\frac{U_t^2(LD)^2}{4\varepsilon^2}
 \leq\frac{2U_t^3+(LD)^6}{12\varepsilon^2}.
\)
Corollary~\ref{cor:hardy-l3} therefore gives
\begin{equation}
 \frac{\mathsf P_T^h}{T}
 \leq\frac1{\varepsilon^2}
 \left\{\frac1T\sum_{t=1}^T\norm{g_t}^6+\frac{(LD)^6}{12}\right\}.
 \label{eq:stationarity-mean-perturbation}
\end{equation}

Now, we apply Lemma~\ref{lem:filtered-energy} with failure
probability $\delta/2$. On the resulting event,
\(
 \frac{\mathsf Z_T}{T}
 \leq\mathfrak Z_{T,\beta,\delta/2}
 \leq\frac32\mathfrak Z_{T,\beta,\delta},
\)
where the second inequality follows from
$\log(8/\delta)\leq\frac32\log(4/\delta)$.
On the sixth-moment event, whose failure probability
is at most $\delta/4$,
\eqref{eq:stationarity-mean-perturbation} and \(C_6\geq1\) give
\(
 \frac{\mathsf P_T^h}{T}
 \leq\frac{13\mathfrak G_\delta}{12\varepsilon^2}.
\)
The union bound shows that their intersection has
probability at least $1-3\delta/4\geq1-\delta$.
On this intersection,
\begin{equation}
 \frac1T\sum_{t=1}^T\norm{\delta_t}^2
 \leq16\kappa_\varepsilon L^2D^2\beta^2
      +6\kappa_\varepsilon \mathfrak Z_{T,\beta,\delta}
      +\frac{13\mathfrak G_\delta}{6\varepsilon^2}.
 \label{eq:stationarity-average-direction-error}
\end{equation}
Substituting into \eqref{eq:stationarity-before-error}
gives
\[
\begin{aligned}
 \frac1T\sum_{t=1}^T
 \norm{\mathcal G_{\alpha,t}(\theta_t)}^2
 \leq{}&
 \frac{8(F_1(\theta_1)-F_\star+\mathcal V_T)}{\alpha T}+96\kappa_\varepsilon L^2D^2\beta^2
  +36\kappa_\varepsilon\mathfrak Z_{T,\beta,\delta}
  +\frac{13\mathfrak G_\delta}{\varepsilon^2}.
\end{aligned}
\]
Absorbing the numerical constants proves
\eqref{thm:thm_4}.
\end{proof}

\newpage

\section{Additional Experiments}

\subsection{Objective specific details}
\textbf{Kronecker quadratic. } The parameter and target are $W,W_k^\star\in\R^{4\times4}$.  The objective is
\begin{equation*}
 F_k(W)=\frac12\operatorname{tr}\!\left[
   (W-W_k^\star)^{\mathsf T}L(W-W_k^\star)R
 \right]
\end{equation*}
with population gradient
\(
 \nabla F_k(W)=L(W-W_k^\star)R.
\)
Independent orthogonal eigenvector matrices for $L$ and $R$ are generated from Gaussian QR factorizations. Both matrices have four geometrically spaced eigenvalues from $1$ to $\sqrt{10}$. Therefore the vectorized Hessian has condition number $10$. This is a strongly convex example of left--right matrix geometry.

\textbf{Huber regression. }
Here $W,W_k^\star\in\R^{4\times4}$ and $w=\operatorname{vec}(W)$. A fixed design $X\in\R^{64\times16}$ has independent $N(0,1)$ entries and. With threshold $\kappa=0.1$ and centered ridge strength $\lambda=0.1$, define
\begin{equation*}
 F_k(w)=\frac1{64}\sum_{i=1}^{64}
 h_\kappa\!\left(x_i^{\mathsf T}(w-w_k^\star)\right)
 +\frac\lambda2\norm{w-w_k^\star}^2,  \qquad \qquad h_\kappa(r)=
 \begin{cases}
  r^2/2,& |r|\le\kappa,\\
  \kappa|r|-\kappa^2/2 & |r|>\kappa.
 \end{cases}
\end{equation*}

Writing $\psi_\kappa(r)=\max\{-\kappa,\min(r,\kappa)\}$, the analytic
population gradient is
\begin{equation*}
 \nabla F_k(w)=\frac1{64}X^{\mathsf T}
 \psi_\kappa\!\left(X(w-w_k^\star)\right)
 +\lambda(w-w_k^\star),
\end{equation*}
where $\psi_\kappa$ is applied entrywise. The ridge makes the objective globally $\lambda$-strongly convex while preserving $w_k^\star$ as its exact minimizer. The Huber score supplies the intended nonquadratic regression geometry.

\textbf{Soft-label logistic regression. }
Here $W,W_k^\star\in\R^{4\times4}$ and $w=\operatorname{vec}(W)\in\R^{16}$. The fixed design $X\in\R^{64\times16}$ has independent
$N(0,1)/\sqrt{16}$ entries. If
$p_k^\star=\operatorname{sigmoid}(Xw_k^\star)$, the soft-label cross-entropy has gradient
\(
 \nabla F_k(w)=\frac1{64}X^{\mathsf T}
 \{\operatorname{sigmoid}(Xw)-p_k^\star\}
\). The parameters were chosen so that the problem remains convex.

\textbf{Population phase retrieval. }
The parameter and target are $W,W_k^\star\in\R^{4\times4}$. Up to an
additive target-only constant, the analytic population objective is
\begin{equation*}
 F_k(W)=\frac34\norm{W}_F^4
 -\frac12\norm{W_k^\star}_F^2\norm{W}_F^2
 -\ip{W}{W_k^\star}_F^2,
\end{equation*}
with gradient
\begin{equation*}
 \nabla F_k(W)
 =\{3\norm{W}_F^2-\norm{W_k^\star}_F^2\}W
 -2\ip{W}{W_k^\star}_F W_k^\star.
 \label{eq:phase-gradient}
\end{equation*}
The objective is nonconvex and has the usual sign symmetry, and chosen to be smooth.

\textbf{Cosine-perturbed matrix objective. }
Here $W,W_k^\star\in\R^{4\times4}$. Gaussian QR generates an orthonormal basis $\{A_j\}_{j=1}^{16}$ of $\R^{4\times4}$ under the Frobenius inner product. With $\mu=0.2$, amplitude $a=0.03$, and frequency $\omega=5$, define
\begin{equation*}
 F_k(W)=\frac\mu2\norm{W-W_k^\star}_F^2
 +a\sum_{j=1}^{16}
 \left[1-\cos\!\left(\omega\ip{A_j}{W-W_k^\star}_F\right)\right].
\end{equation*}
Its analytic gradient is
\begin{equation*}
 \nabla F_k(W)=\mu(W-W_k^\star)
 +a\omega\sum_{j=1}^{16}
 \sin\!\left(\omega\ip{A_j}{W-W_k^\star}_F\right)A_j.
\end{equation*}
Every summand is nonnegative and vanishes at $W_k^\star$, so the moving target is a known global minimizer. Nevertheless, $a\omega^2=0.75>\mu$, and along a basis direction at cosine phase $\pi$ the curvature is $\mu-a\omega^2=-0.55$. The objective is therefore genuinely smooth and nonconvex.

\textbf{Low-rank matrix factorization. }
The parameters are $U,V\in\R^{12\times3}$. With target $M_k^\star=U_k^\star(V_k^\star)^{\mathsf T}$,
\(
 F_k(U,V)=\frac12\norm{UV^{\mathsf T}-M_k^\star}_F^2.
\)
Writing $E_k=UV^{\mathsf T}-M_k^\star$, the analytic gradients are
\(
 \nabla_UF_k=E_kV,
 \,
 \nabla_VF_k=E_k^{\mathsf T}U.
\)
This objective is nonconvex and nonidentifiable under factor rotations.

\subsection{Experiment details}
\textbf{Drift and noise schedule. } For matrix-valued targets, the drift is generated through normalized Gaussian directions:
\[
W_{t+1}^\star = W_t^\star + \Delta_t\frac{U_t}{\|U_t\|_F}, \qquad (U_t)_{ij}\sim N(0,1).
\]
At each round the optimizer receives
\(
G_t=H_t(W_t)+ \sigma_t \,Z_{t,s},
\)
where $H_t$ is the analytic population gradient and every coordinate of $Z_{t,s}$ is an independent Rademacher sign multiplied by $1/\sqrt d$. For a multi-block parameter, $d$ is the total number of coordinates across all blocks. Hence $\E[Z_{t,s}]=0$, $\norm{Z_{t,s}}=1$, and $\norm{G_t-H_t(W_t)}^2=\sigma_t^2$ at every round. The dimensions are $d=16$ for the five single-matrix
objectives and $d=72$ for matrix factorization. Thus, $\Delta_t$ sets the drift magnitude and $\sigma_t^2$ decides the stochastic noise level.

\textbf{Optimizers and tuning. } EMA matrix-Shampoo is implemented as in Algorithm~\ref{alg:shampoo}. For real data analyses SGD, bias corrected momentum SGD, and bias corrected momentum Adam with ridge parameter have been used. The effective stepsize \(\alpha\) is tuned from a problem-specific grid, while the ridge parameter \(\varepsilon\) is tuned over the grid
\(
\{10^{-4},10^{-3},10^{-2},10^{-1},1,10\}.
\) Curves are then reported for the selected learning rate, averaged over seeds $\{0,1,2,3,4,5,6,7,8,9\}$, with shaded regions denoting one standard error of the mean. 

\textbf{Evaluation. } For convex problem we evaluate the average tracking error $\frac1t\sum_{k=1}^t\|\widehat W_k - W_k^\star\|^2_F$, and for non-convex objectives we evaluate the average stationary gap $\frac1t\sum_{k=1}^t \left\| \mathcal G_{\alpha,k}(\theta_k)  \right\|^2$; see \eqref{thm:thm_4}, over $T = 5000$ iterations.

All other experimental details can be found in the supplementary material.

\subsection{Real-data experimental details}
\label{app:real-data-details}

\paragraph{Common setup.}
We compared SGD, Momentum SGD, Adam, and EMA-Shampoo using an end-to-end ResNet-18. The convolutional backbone was initialized from ImageNet-pretrained weights, while the final linear classifier was initialized to zero. The backbone and classifier were optimized jointly; batch-normalization running statistics and affine parameters were frozen to avoid introducing an additional source of temporal state. The penultimate feature vector was $\ell_2$-normalized before classification.

We minimized multiclass cross-entropy with the explicit penalty $\frac{\lambda}{2}\lVert\theta\rVert_2^2$, where $\lambda=10^{-4}$ for every method. No additional optimizer weight decay was used. The batch size was 64. For each seed, all methods used the same initialization, preprocessing, temporal ordering, and minibatch order. Both the model parameters and the complete optimizer state, including momentum and preconditioner statistics, were carried across temporal shifts. We did not reset the optimizer or use replay. Results were computed with numerous seeds and are reported as the mean, with shaded regions indicating one standard deviation.

\paragraph{CLEAR-10 protocol.}
CLEAR-10 contains ten named visual concepts and a background class, giving 11 prediction classes. We used the official paired training and test splits from temporal buckets 1--10. Images were resized to 256 pixels, center-cropped to $224\times224$, and normalized using the ImageNet statistics. Bucket 1 was used as an initialization period for 10 epochs and was excluded from the reported stream. The model was then trained for one pass over each of buckets 2--10 in chronological order. Training and test errors in the figure are evaluated at the end of each bucket using the same post-update model state.

\paragraph{CORe50 protocol.}
Our CORe50 experiment is a custom session-ordered, object-disjoint protocol, rather than one of the standard CORe50 continual-learning scenarios. The prediction labels are the ten official object \emph{categories}, not the 50 physical-object identities. Each category contains five physical objects. We used object positions 1--4 in every category for training, giving 40 training objects, and position 5 for testing, giving 10 held-out objects. Consequently, all ten test categories are represented during training, but the physical instances used for testing are never observed during training; the experiment therefore measures category generalization to unseen objects rather than classification of unseen classes.

Sessions $s1,\ldots,s11$ were processed in numerical order. We retained 50
uniformly spaced frames per object and session, resulting in 2,000 training
and 500 held-out test images per session (22,000 and 5,500 images in total,
respectively). Images were resized to 144 pixels, center-cropped to
$128\times128$, and normalized using the ImageNet statistics. Session $s1$
was used for 10 initialization epochs, after which sessions $s2$--$s11$ were
processed once each. The acquisition sessions induce changes in illumination,
background, viewpoint, object pose, hand configuration, and occlusion. As in
CLEAR-10, the displayed values are the training and held-out errors obtained
from the end-of-session model states.

\paragraph{Optimizer settings and model selection.}
The first-moment coefficient was varied over
$\beta_1\in\{0.5,0.99\}$ for SGDm, Adam, and Shampoo. The remaining
settings were fixed across $\beta_1$ and seeds. On both dataset, the learning
rate and adaptive step-size for SGD, SGDM, Adam, and Shampoo were tuned optimally,  with
$\beta_2=0.999$ for Adam and Shampoo.

\paragraph{Shampoo implementation.}
Each parameter tensor was treated independently. A tensor of order at least
two was reshaped as
$\mathrm{shape}[0]\times\prod_{j\geq1}\mathrm{shape}[j]$; vectors were treated
as column matrices. The resulting matrices were partitioned along both axes
into blocks of at most $128\times128$, with separate left and right
second-moment statistics for each block. These statistics were updated at
every optimizer step, while their inverse fourth roots were recomputed every
10 steps and cached between recomputations. First- and second-moment estimates
were bias corrected. We added the permanent matrix ridge
$\varepsilon=10^{-4}$ before computing each inverse fourth root.

\subsection{Empirical minimax scaling}

We further examine the minimax tracking rate on objectives not shown in the main text. Figure~\ref{fig:minimax-scaling_app} reports the average tracking error normalized by
\(
r_T^\star
=
\sigma^{4/3}\mathcal D_T^{1/3}/(\mu^{4/3}T^{1/3})
\)
under
\(
\mathcal D_T\asymp \log T,\,
T^{1/3},\,
T^{1/2},
\)
and \(T\). While Figure~\ref{fig:minimax-scaling} considers the Kronecker, pseudo-Huber, and soft-label logistic objectives, Figure~\ref{fig:minimax-scaling_app} considers phase retrieval and the cosine objective. Across these additional problems, the normalized errors remain broadly stable over the horizon, providing further empirical support for the predicted minimax scaling.

\begin{figure}[H]
    \centering
    \includegraphics[width=\columnwidth]
    {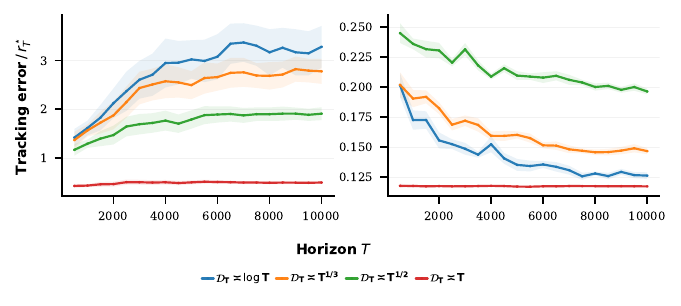}
    \caption{\textbf{Empirical validation of the minimax tracking rate.}
    Results for \textbf{phase retrieval} (left) and the \textbf{cosine objective} (right) under the path-variation regimes
    \(\mathcal D_T\asymp \log T,\;T^{1/3},\;T^{1/2}\), and \(T\).
    The average tracking error is normalized by
    \(r_T^\star=
    \sigma^{4/3}\mathcal D_T^{1/3}/(\mu^{4/3}T^{1/3})\).
    The normalized errors remain overall stable over the horizon, supporting the predicted minimax scaling.}
    \label{fig:minimax-scaling_app}
\end{figure}

\subsection{Ridge--geometry tradeoff}

We next complement the ridge experiment in the main text with the remaining objectives. Figure~\ref{fig:ridge-geometry_app} considers the Kronecker, soft-label logistic, and cosine objectives, whereas Figure~\ref{fig:ridge-geometry} in the main text reports pseudo-Huber regression, matrix factorization, and phase retrieval. For each problem, we compare the tracking error or projected-stationarity gap with the preconditioned filtered-noise energy
\(
\mathfrak Z_T^Q(\varepsilon)
=
T^{-1}\sum_{t=1}^T
\|Q_t(\varepsilon)z_t\|^2
\)
across high-drift/low-noise and low-drift/high-noise regimes. The complementary behavior across the two regimes illustrates the stability--geometry tradeoff induced by the ridge parameter.

\begin{figure}[!htbp]
    \centering
    \includegraphics[width=\columnwidth]
    {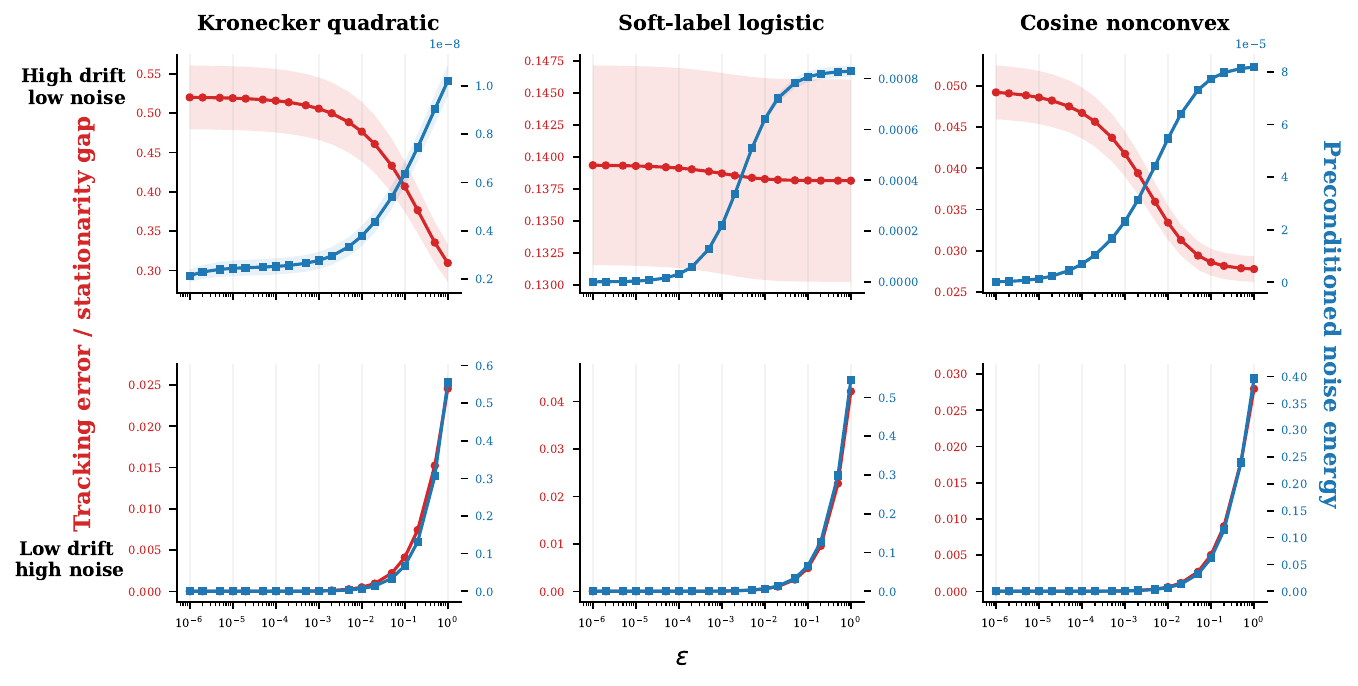}
    \caption{\textbf{Ridge stability--geometry tradeoff.}
    Average tracking error for convex objectives and average projected-stationarity gap for nonconvex objectives. In the high-drift, low-noise regime (top), the preconditioned filtered-noise energy
    \(\mathfrak Z_T^Q(\varepsilon)\) remains small, and the decreasing error with increasing \(\varepsilon\) is consistent with dominance of the preconditioner-perturbation contribution
    \(\mathfrak G_\delta/\varepsilon^2\).
    In the low-drift, high-noise regime (bottom),
    \(\mathfrak Z_T^Q(\varepsilon)\) becomes increasingly important as \(\varepsilon\) grows, producing the complementary behavior.}
    \label{fig:ridge-geometry_app}
\end{figure}

\subsection[Drift-noise tradeoff and optimal beta1]{Drift-noise tradeoff and optimal \(\beta_1\)}
The main-text experiments already illustrate the drift--noise tradeoff in the choice of first-moment memory and the corresponding variation in the preferred \(\beta_1\) across high-drift/low-noise and low-drift/high-noise regimes. Here, we provide additional trajectory-level comparisons of Shampoo across objectives and optimizers.

\begin{figure}[!htbp]
    \centering
    \includegraphics[width=\columnwidth]
    {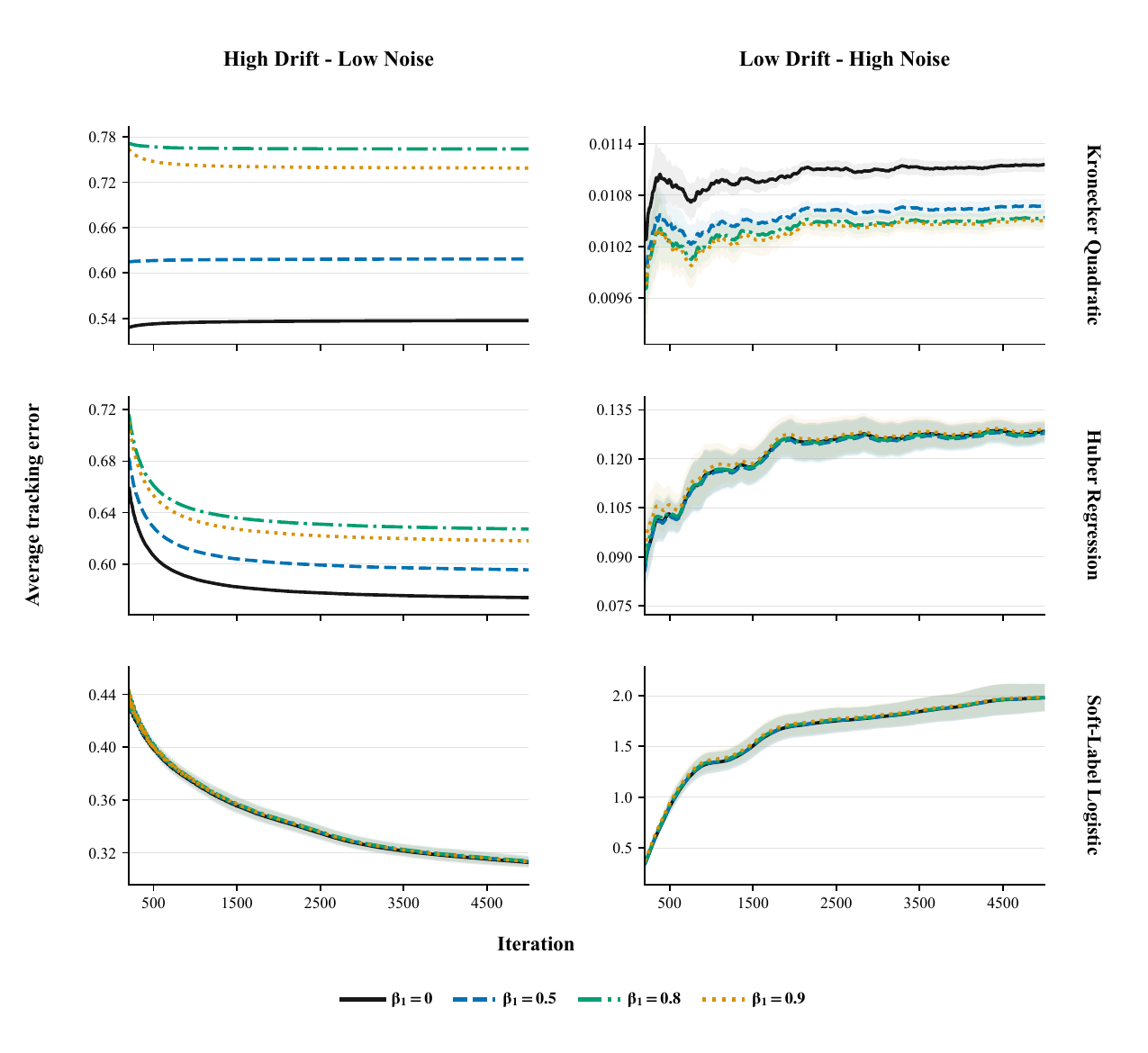}
   \caption{\textbf{Shampoo tracking trajectories for convex objectives.} Results for \textbf{Kronecker quadratic}, \textbf{Huber}, and \textbf{soft-label logistic} (top to bottom) under high drift/low noise $(\Delta,\sigma^2)=(0.15,0.01)$ (left) and low drift/high noise $(0.003,1)$ (right); fixing $\beta_2=0.999$. Curves show the ten-seed mean cumulative average tracking error, with one-standard-error bands. Overall, strong drift generally favors shorter first-moment memory, whereas higher noise can favor moderate memory, with the size of this benefit depending on the objective.}
    \label{fig:appendix:beta1_convex}
\end{figure}

\begin{figure}[!htbp]
    \centering
    \includegraphics[width=\columnwidth]
    {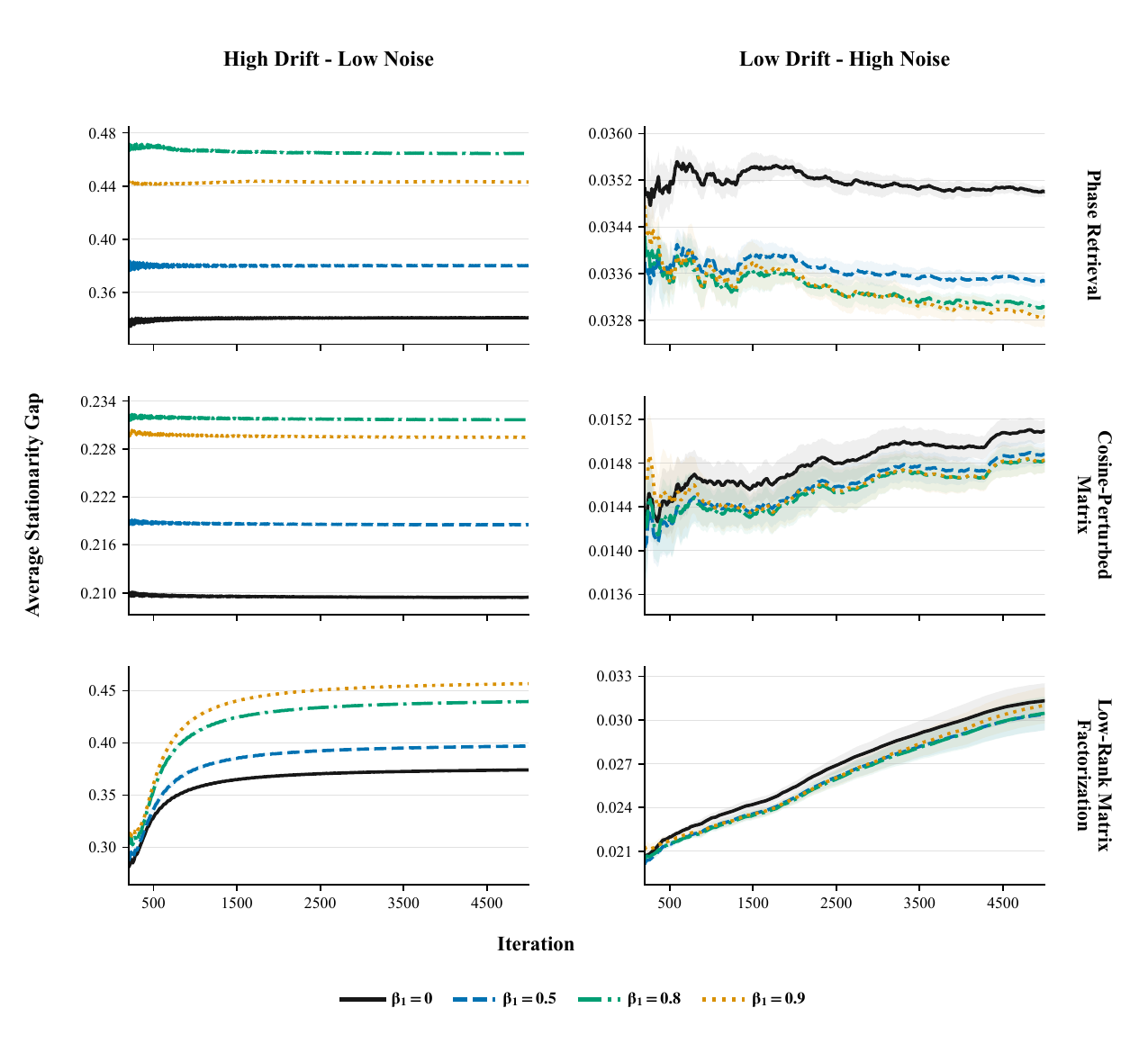}
   \caption{\textbf{Shampoo tracking trajectories for non-convex objectives.} Results for \textbf{Phase retrieval}, \textbf{cosine-perturbed matrix}, and \textbf{low-rank matrix factorization} (top to bottom) under high drift/low noise $(\Delta,\sigma^2)=(0.15,0.01)$ (left) and low drift/high noise $(0.003,1)$ (right); fixing $\beta_2=0.999$. Curves show the ten-seed mean cumulative average stationary gap, with one-standard-error bands. Overall, strong drift generally favors shorter first-moment memory, whereas higher noise can favor moderate memory, with the size of this benefit depending on the objective.}
    \label{fig:appendix:beta1_nonconvex}
\end{figure}

\end{document}